\documentclass[11pt]{article}

\usepackage[truedimen,margin=25truemm]{geometry}

\usepackage{graphicx} 
\usepackage{epsfig} 
\usepackage{subfigure} 

\usepackage{amsmath}
\usepackage{amsthm}
\usepackage{amssymb}
\usepackage{algorithm}
\usepackage{algorithmicx}
\usepackage{algpseudocode}
\usepackage{algpseudocode}

\usepackage{bm}
\usepackage{bbm}

\usepackage{color}

\usepackage{ulem}

\usepackage{booktabs}
\usepackage{multirow}

\usepackage{hyperref}

\usepackage{graphicx} 
\usepackage{epsfig}
\usepackage{subfigure}

\usepackage{amssymb}
\usepackage{amsmath}
\usepackage{amsthm}
\usepackage{mathtools}
\usepackage[raggedright]{titlesec} 

\def\seq#1{\mbox{\boldmath $#1$}}
\def\vec#1{\mbox{\boldmath $#1$}}
\def\mat#1{\mbox{\bf #1}}

\algrenewcommand\algorithmicrequire{\textbf{Input:}}
\algrenewcommand\algorithmicensure{\textbf{Output:}}

\definecolor{lightgray}{rgb}{0.65, 0.65, 0.65}
\newcommand{\CBLG}[1]{\!\!{\textbf{#1}\!\!}}
\definecolor{lightlightgray}{rgb}{0.86, 0.86, 0.86}

\newtheorem{definition}{Definition}
\newtheorem{lemma}{Lemma}
\newtheorem{proposition}{Proposition}
\newtheorem{theorem}{Theorem}

\newcommand{\argmin}{\mathop{\rm arg~min}\limits}

\title{LCS Graph Kernel Based on Wasserstein Distance in\\Longest Common Subsequence Metric Space}

\author{Jianming Huang\thanks{Department of Computer Science and Communications Engineering, Graduate School of Fundamental Science and Engineering, WASEDA University, 3-4-1 Okubo, Shinjuku-ku, Tokyo 169-8555, Japan (e-mail: koukenmei@toki.waseda.jp) } \and Zhongxi Fang\thanks{Dept. of Computer Science and Engineering, School of Fundamental Science and Engineering, WASEDA University, 3-4-1 Okubo, Shinjuku-ku, Tokyo 169-8555, Japan (e-mail: fzx@akane.waseda.jp)} \and Hiroyuki Kasai\thanks{Department of Communications and Computer Engineering \& Department of Computer Science and Engineering, WASEDA University, 3-4-1 Okubo, Shinjuku-ku, Tokyo 169-8555, Japan (e-mail: hiroyuki.kasai@waseda.jp)} \thanks{H. Kasai was partially supported by JSPS KAKENHI Grant Numbers JP16K00031 and JP17H01732, and by Support Center for Advanced Telecomm. Technology Research (SCAT).}}

\begin{document}

\maketitle

\begin{abstract}
For graph learning tasks, many existing methods utilize a message-passing mechanism where vertex features are updated iteratively by aggregation of neighbor information. This strategy provides an efficient means for graph features extraction, but obtained features after many iterations might contain too much information from other vertices, and tend to be similar to each other. This makes their representations less expressive. Learning graphs using paths, on the other hand, can be less adversely affected by this problem because it does not involve all vertex neighbors. However, most of them can only compare paths with the same length, which might engender information loss. To resolve this difficulty, we propose a new Graph Kernel based on a Longest Common Subsequence (LCS) similarity. Moreover, we found that the widely-used $\mathcal{R}$-convolution framework is unsuitable for path-based Graph Kernel because a huge number of comparisons between dissimilar paths might deteriorate graph distances calculation. Therefore, we propose a novel metric space by exploiting the proposed LCS-based similarity, and compute a new Wasserstein-based graph distance in this metric space, which emphasizes more the comparison between similar paths. Furthermore, to reduce the computational cost, we propose an adjacent point merging operation to sparsify point clouds in the metric space. The source code is available at \url{https://github.com/AkiraJM/LCSGK}.
\end{abstract}

\begin{center}
{\small
Published in Signal Processing: https://doi.org/10.1016/j.sigpro.2021.108281 \cite{Huang_SigPro_2021}.
}
\end{center}

\section{Introduction}
\label{Sec:Intro}

Graph-structured data have been used widely in various fields such as chemoinformatics, bioinformatics, social networks, and computer vision \cite{vishwanathan2010graph,Kriege_ANS_2020}.{ Graph classification is a branch of graph-structured data research aimed at classifying graphs with or without attributes. Many studies have examined graph classification, including investigations of Graph Kernels, which are kernel functions measuring similarity between graphs, and Graph Neural Networks (GNN), which are applications of neural networks in a graph domain.} For graph classification tasks, a key difficulty is finding {an effective method of exploring graphs and aggregating the extracted information, which can be adapted to different classification criteria \cite{huang2021graph}. Many state-of-the-art studies of GNNs and Graph Kernels use message-passing mechanisms based on a vertex neighborhood. This aggregation strategy brings great success to graph classification, but it still entails some shortcomings such as {\it over-smoothing} difficulties \cite{li2018deeper, chen2020measuring} as described below.}

\subsubsection{Breadth-first strategy and Depth-first strategy}
Various strategies are available for {graph exploration and feature aggregation}. Some studies specifically find relations between vertices and their neighbors. This type of strategy is a {\it breadth-first strategy}, which tends to construct a neighbor network surrounding a certain vertex. For an immediate instance, the Weisfeiler--Lehman Graph Kernel \cite{Shervashidze_JMLR_2011_s}, uses the theory of the Weisfeiler--Lehman test \cite{weisfeiler1968reduction}, where vertex neighbors are aggregated to produce a new vertex label using a hash mapping function. The new label denotes a certain neighborhood pattern. Then graph similarity can be computed through comparison of these compressed labels of topological patterns.
However, such breadth-first strategies entail several shortcomings.
(1) Because {many} breadth-first strategies aggregate information of all neighbors, the embedding of a vertex usually combines excessive information, especially in a dense graph. (2) As the depth of a neighbor network increases, {the number of neighbor vertices} grows geometrically, as does the quantity of information, especially when one wants to assess the relations of vertices that are mutually distant. {This phenomenon is called over-smoothing because each aggregation of neighborhood is a ``smooth'' operation: as the depth increases, the aggregated feature of all vertices tend to be similar, leading to over-smoothing \cite{li2018deeper, chen2020measuring}.}

Along other avenues of development, some studies have specifically examined paths and walks of graphs. In contrast to breadth-first strategies, they learn graph topology by comparing paths or walks. In immediate instances, one can find explanations of random walk Graph Kernel \cite {gartner2003graph} and shortest path Graph Kernel \cite {borgwardt2005shortest} in the literature. We regard these as {\it depth-first strategies}, which use a simple form of subgraph to learn the subgraph structure, instead of using complicatedly connected neighbor networks. In this manner, comparison of neighbor networks is transformed to a difficulty of comparing sets of path (or walks). Adverse effects of the problem of excessive information such as breadth-first strategies are less likely {because they usually aggregate fewer node features than breadth-first methods when depth grows. Earlier work in \cite{li2018deeper} also uses a random walk to help explore graphs for GNN.} Nevertheless, this category of approach, called a {\it path-based} method or a {\it walk-based} method, is adversely affected by the common truth that they entail high computational costs because sampled paths are usually numerous and have high redundancy.

\subsubsection{Path comparison strategy based on longest common subsequence}
In the path-based and walk-based methods, a traditional strategy to compare two path sequences is to judge whether these two paths are completely identical. Taking the random walk kernel for example, two random walks are judged as identical if their lengths are equal, and also if their label sequences are identical for labeled graphs. 
However, 
most such methods require that the compared paths should have {\it equal} length. Therefore, because those paths with different length cannot be compared under this approach, it might cause information loss. 
To this end, we propose to leverage a {\it longest common subsequence} (LCS) to calculate path similarity, which enables comparison of the paths having different length.
Furthermore, to alleviate increase of computational cost of many pairs of paths, we propose a merging approach to combine similar and redundant paths. This merging is achieved in the proposed metric space, which is elaborated on below.

\subsubsection{Application of optimal transport theory}
Because of computational improvements in solving optimal transport (OT) problems in recent years, optimal transport theory \cite{Villani_2008_OTBook_s, Peyre_2019_OTBook_s} has been used in many machine-learning domains including graph classification tasks. A representative product of those examinations is Wasserstein--Weisfeiler--Lehman Graph Kernels (WWL) \cite{togninalli2019Wasserstein}, which improved the Weisfeiler--Lehman kernel by replacing a traditional $\mathcal{R}$-Convolution framework \cite{haussler1999convolution} with an optimal transport scheme. Most earlier Graph Kernel studies use the $\mathcal{R}$-Convolution framework, which tends to decompose a graph into several substructures such as subgraphs, paths, vertices, and edges. Then the overall graph similarity is computed by aggregating similarity between substructures. 
Let $G_1, G_2$ be two graphs to be compared, then the value of $\mathcal{R}$-Convolution Graph Kernel can be written as
\begin{eqnarray}
\label{Eq:RC}
k(G_1, G_2) = \sum_{x\in R^{-1}(G_1), y\in R^{-1}(G_2)}\prod_{d=1}^D k_d(x_d, y_d)
\end{eqnarray}
\cite{haussler1999convolution}. { Let $\mathcal{R} = \mathcal{R}_1 \times ... \times \mathcal{R}_D$ be a substructure space for graph $G$ where $\mathcal{R}_1, ..., \mathcal{R}_{D}$ are $D$ nonempty, separable metric spaces, and} graph $G$ can be decomposable into elements of $\mathcal{R}$, for example, vertices and paths. Let $R(\cdot)$ be the function by which $R(G) = x$ holds only if $x\in \mathcal{R}$ makes up $G$. Then $R^{-1}(G) = \{ x\in \mathcal{R}: R(G)=x \}$. $k_d$ is a kernel function for substructures in $\mathcal{R}_d$.
Nevertheless, traditional aggregation strategies such as pure sum, max, and min are overly simplistic. They might be unable to capture complex features of graph structures {\cite{togninalli2019Wasserstein}}. Moreover, we found that the $\mathcal{R}$-Convolution framework is unsuitable for the path-based strategy for the following reason: most graphs include a larger number of paths than vertices, and most of these paths are markedly different from each other and are unlikely to be matched. Although the kernel value $k_d$ of such {\it dissimilar} paths is very small in general, it still becomes non-negligible when there exist a huge number of these noisy values. Consequently, Eq.(\ref{Eq:RC}) will be easily affected by these numerous dissimilar pairs, leading to deterioration of the calculation of graph distances. Consequently, we propose to compute the kernel value using the Wasserstein distance with a new ground metric based on the proposed LCS similarity. The main motivation is that the Wasserstein distance based on optimal transport theory can assign less importance to such dissimilar paths between two graphs. As a consequence, the estimated similarity metric is expected to become more accurate than the one obtained from the existing $\mathcal{R}$-Convolution framework.

\begin{figure}[t]
\setlength{\belowcaptionskip}{-0.05cm} 
\centering
\includegraphics[width=0.6\textwidth]{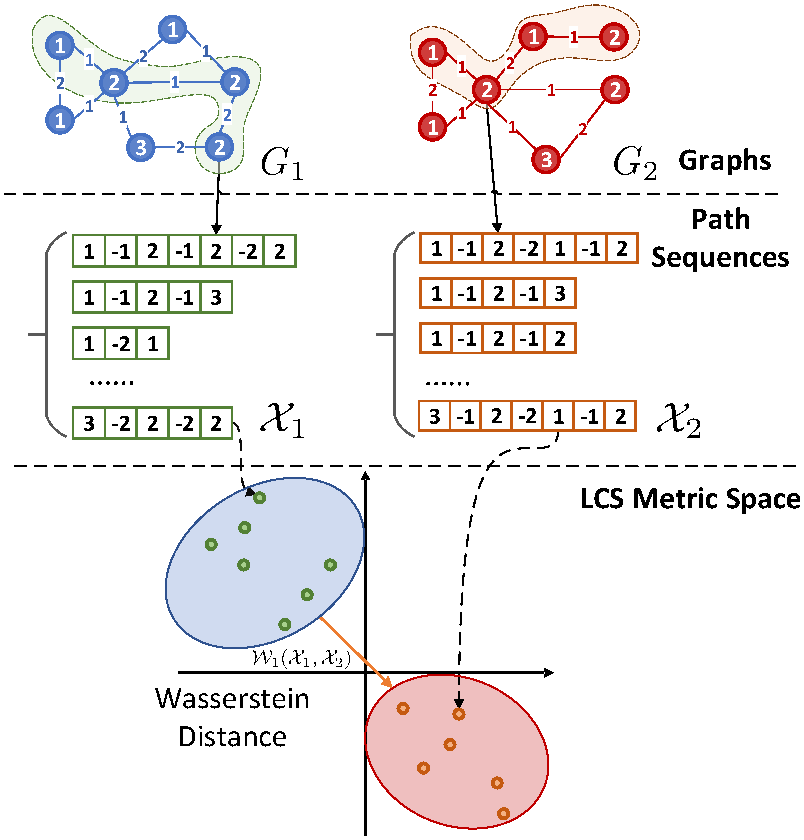} 
\caption{Visual summary of the proposed LCS Graph Kernel. Shortest paths are serialized as label sequences and are mapped to an LCS metric space. The Wasserstein distance of two distributions in this space is used as the LCS kernel value.
} 
\label{Fig.1} 
\end{figure}

\subsubsection{Overview of proposed method}
This paper specifically addresses the path-based methods. To overcome such methods' restriction of the path length, and to compute a similarity metric that is less affected by dissimilar paths, we propose the LCS Graph Kernel using the LCS-based similarity and Wasserstein distance. Figure \ref{Fig.1} portrays a summary of the proposed LCS Graph Kernel. With our method, we first perform shortest path serialization to transform all vertex-to-vertex shortest paths as path sequences. We define a shortest path sequence set to represent the graph structure. We then map these path sequence sets to an LCS metric space, where a metric between two sequences is based on the length of their longest common subsequence. As the final step, we use the 1-Wasserstein distance of two distributions as our kernel value. In addition to our basic method, we propose a faster variant that greatly reduces computational costs by newly introducing an adjacent point merging operation to sparsify points clouds in the LCS metric space. {We evaluate our proposed methods using several widely assessed small graph benchmark datasets and a large-scale graph dataset, comparing to other state-of-the-art Graph Kernel methods, and Graph Neural Network methods: a hot research topic in graph learning. The evaluation results clarify that our proposed methods are competitive and that they outperform some state-of-the-art Graph Neural Network methods even in large-scale dataset.} Generally, our contributions can be summarized as described below.
\begin{itemize}
	\item We present a novel Graph Kernel to compare labeled graphs using the LCS similarity of path sequences and optimal transport theory. This method enables paths with different length to be compared, and is robust against noise caused by a huge number of pairs of dissimilar paths, which gives a more accurate metric than the $\mathcal{R}$-convolution framework does.
	\item We conduct numerical experiments using popular small and large-scale real-world datasets, with results demonstrating that our proposed method outperforms many state-of-the-art methods in graph classification tasks.
\end{itemize}

{All theorem proofs, all the pseudo-code of the proposed algorithms and additional experiment results are presented in Supplementary Materials.}

\section{Preliminaries}
This section first introduces some notation and preliminary points of graphs. Hereinafter, we represent scalars as lower-case letters $(a, b, \ldots)$, vectors as bold typeface lower-case letters $(\vec{a}, \vec{b}, \ldots)$, and matrices as bold typeface capitals $(\mat{A}, \mat{B}, \ldots)$. An element at $(i,j)$ of a matrix \mat{A} is represented as $\mat{A}(i,j)$. $\vec{1}_{n}=(1, 1, \ldots, 1)^T \in \mathbb{R}^n$. We represent sequences as bold italic typeface lower-case $(\seq{a},\seq{b},...)$, and $|\seq{a}|$ denotes the length of sequence $\seq{a}$. We express $\mathbb{R}_{+}^{{n}}$ to denote non-negative $n$-dimensional vector. In addition, $\mathbb{R}_{+}^{n \times m}$ represents a nonnegative matrix of size ${n \times m}$. The unit-simplex, simply called {\it simplex}, is denoted by $\Delta_n$, which is the subset of $\mathbb{R}^n$ comprising all nonnegative vectors for {the elements of which sum up to 1}. In addition, $\delta_x$ is the Dirac function at $x$.

\subsection{Preliminaries of graphs}
\label{Sec:GM}

A graph is a pair $G=(\mathcal{V}, \mathcal{E})$ consisting of a set of $n$ vertices (or nodes) $\mathcal{V}=\{v_1, v_2, \ldots, v_n\}$ and a set of $m$ edges $\mathcal{E} \subseteq \mathcal{V} \times \mathcal{V}$. {If a graph includes only edges with no direction, then we say it an undirected graph}. 
The numbers of vertices and edges are, respectively, $|\mathcal{V}|$ and $|\mathcal{E}|$. If two vertices, say $v_i, v_j \in \mathcal{V}$, are connected by an edge $e$, then this edge is denoted as $e_{ij}$. These two vertices are said to be adjacent or neighbors. {For all assumptions in this paper,} we consider only undirected graphs with no self-loop. 

Given an undirected graph $G=(\mathcal{V}, \mathcal{E})$ and a vertex $v_i \in \mathcal{V}$, the degree of $v_i$ in $G$, denoted as $\sigma_i$, is the number of edges incident to $v_i$. It is defined as 
$\sigma_i =  |\{ v_j: e_{ij} \in \mathcal{E}\}| = | \mathcal{N}(v_i)|,$
where $\mathcal{N}(v_i)$ represents the neighborhood set of $v_i$.

A walk in a graph $G=(\mathcal{V}, \mathcal{E})$ is a sequence of vertices $v_1, v_2, . . ., v_{k+1}$, where $v_i \in \mathcal{V}$ for all $1 \leq i \leq k+1$ and {$(v_i, v_{i+1}) \in \mathcal{E}$} for all $1 \leq i \leq k$. The walk length is equal to the number of edges in the sequence, i.e., $k$ in the case above. A walk in which $v_i \neq v_j \Leftrightarrow i \neq j$ is called a path. {A path between two adjacent vertices is equivalent to an edge}. A set of paths between two non-adjacent vertices is denoted as $\mathcal{P}_{i,j}$. One path of $\mathcal{P}_{i,j}$ is denoted as $p_{i,j}$. Moreover, the length of $p_{i,j}$ is denoted as $|p_{i,j}|$.
A shortest path, denoted as $p^{\star}_{i,j}$, from vertex $v_i$ to vertex $v_j$ of a graph $G$, is a path from $v_i$ to $v_j$ such that no other path exists between these two vertices with smaller length.

\subsection{Longest common subsequence}
\label{Sec:LCS}

A sequence is defined mathematically as an enumerated collection of objects of a certain order, where repetition of elements is allowed. The length of a sequence denotes the number of its elements. Given a sequence with length $n$, it can be written as $\seq{x} = (x_i)_{i = 1}^n$, where $x_i$ denotes the $i$-th element.
A subsequence of a given sequence is derived by removing some elements of the original sequence under the premise that the order of the remaining elements is retained as the original sequence. For example, given a sequence of $(1,2,3,4,5)$, both $(2,3,4)$ and $(1,3,5)$ are subsequences of the original sequence. The definition of the longest common subsequence problem can be presented as shown below \cite{wagner1974string, cormen2009introduction}.
\begin{definition}[Longest Common Subsequence Problem]
\label{Def:LCSP}
The goal of a longest common subsequence problem is to find a sequence $\seq{x}^\star$ that satisfies the following conditions: (1) $\seq{x}^\star$ is a subsequence of each sequence in a set of sequences $\mathcal{X} = \left\{\seq{x}_1,\seq{x}_2...,\seq{x}_n\right\}$. (2) $\seq{x}^\star$ is the longest among all subsequences that satisfy condition (1). Then we consider $\seq{x}^\star$ as the longest common subsequence in $\mathcal{X}$.
\end{definition}

If the number of elements in objective set of sequences $\mathcal{X}$ in this problem is constant, then the solution of the problem is solvable in polynomial time using dynamic programming. Under a situation in which there are only two sequences in $\mathcal{X}$, the computational complexity is $O(n\times m)$ \cite{maier1978complexity}, where $n,m$ denotes the length of each sequence.

\subsection{Optimal transport (OT)}
\label{Sec:GW}
Optimal transport provides a means of comparing probability distributions, which are histograms in the finite dimensional case \cite{Villani_2008_OTBook_s, Peyre_2019_OTBook_s}. This quantity is defined over the {\it same ground space} or multiple pre-registered ground spaces. This comparison is interpreted as a {\it mass movement problem}, which seeks an optimum plan to move the mass from one distribution to the other at minimal cost. Actually, OT is rapidly gaining popularity in a multitude of machine learning problems, ranging from low-rank approximation \cite{Seguy_NIPS_2015_s}, to dictionary learning \cite{Rolet_AISTATSI_2016_s}, domain adaptation \cite{Courty_ECML_2014_s}, clustering \cite{Cuturi_ICML_2014_s}, semi-supervised learning \cite{Solomon_TOG_2016_s}, color transfer \cite{Fukunaga_arXiv_2021}, linear discriminative learning \cite{Kasai_ICASSP_2020}, and clustering \cite{fukunaga_ICPR_2020}.

We define two simplexes of histograms with $n_1$ and $n_2$ in the same metric space, which are defined, respectively, as 
$\Delta_{n_1} =\{\vec{p} \in \mathbb{R}_{{+}}^{n_1}; \sum_i p_i=1\}$, and 
$\Delta_{n_2} =\{\vec{q} \in \mathbb{R}_{{+}}^{n_2}; \sum_j q_j=1\}$, where in {\it mass movement problem}, $\vec{p},\vec{q}$ are usually regarded as mass vectors of each histogram, with elements denoting the mass of each bin. Subsequently, we define two probability measures as
$\mu=\sum_{i=1}^{n_1} p_i \delta_{x_i}$, and $\nu=\sum_{j=1}^{n_2} q_j \delta_{y_j}$,
where $x_i \neq x_j$ for $i\neq j$ is assumed without loss of generality. We also consider the {\it ground cost} matrix $\mat{C} \in \mathbb{R}_+^{n_1 \times n_2}$, where $\mat{C}(i,j)$ represents the transportation cost between the $i$-th and $j$-th element. The optimal transport problem between these two histograms is defined as 
\begin{eqnarray*}
\label{Eq:Tstar}
	\mat{T}{^*}(\mat{C},\vec{p},\vec{q}) 
	=\argmin_{\scriptsize \mat{T}\ \in\ \mathcal{U}_{n_1 n_2}}
	\langle \mat{T}, \mat{C}\rangle, 	
\end{eqnarray*}
where $\mathcal{U}_{n_1 n_2}$ is defined as 
\begin{eqnarray*}
\label{Eq:}
	\mathcal{U}_{ n_1n_2} := 
	\left\{
	\mat{T} \in \mathbb{R}^{n_1\times n_2}_+ : \mat{T} \vec{1}_{n_2} 
	= \vec{p},\ \mat{T}^T \vec{1}_{n_1} = \vec{q}
	\right\},
\end{eqnarray*}	
and $\mathcal{U}_{ n_1n_2}$ represents the polytope of $n_1\! \times\! n_2$ nonnegative matrices such that their row and column marginals are, respectively, equal to $p_i$ and $q_j$. This minimization problem involves a linear programming (LP) problem, i.e., a convex optimization problem. Then, the Wasserstein distance between two measures, denoted as $\mathcal{W}(\mu, \nu)$, is equal to the total distance traversed by the mass under the optimal transport plan $\mat{T}^*$.
Furthermore, by adding an entropic regularizer $H(\mat{T})=-\sum_{i=1}^{n_1} \sum_{j=1}^{n_2} \mat{T}(i,j)(\log (\mat{T}(i,j))-1)$, a solution of the entropically regularized optimal transport problem is solvable efficiently using Sinkhorn's fixed-point iterations \cite{Sinkhorn_PJM_1967_s,Cuturi_2013_NIPS_s,Benamou_2015_SIAMJSC_s}. If no prior information is known about a space, then 
we set $\vec{p}=\frac{1}{n_1} \vec{1}_{n_1}$ and $\vec{q}=\frac{1}{n_2} \vec{1}_{n_2}$. 

\section{Related Work}

For graph classification tasks, Graph Kernel methods have been used widely for several decades. They are also developing rapidly in recent years. Graph kernels are kernel functions that compute the similarity between two graphs. Graph kernel methods have shown effective performances in graph classification tasks using machine learning algorithms. In recent years, as effective performances of optimal transport theory in a machine learning domain, Graph Kernel methods are also improved greatly when combined with optimal transport theory. For now, Graph Kernel methods can be generally divided into two categories as methods of (1) Traditional Graph Kernel and (2) OT-based Graph Kernel.

The first of those classifications of methods include traditional Graph Kernels, most of which are based on the $\mathcal{R}$-convolution framework {\cite{Kriege_ANS_2020}}. To compute the similarity between graphs in various data mining tasks, the random walk kernel \cite{gartner2003graph} has been developed and used widely as an important tool for graph classification. This method is based on the counting of matching random walks in two graphs with a label or not. However, the random walk kernel faces a difficulty by which the computational cost is $O(n^6)$ for comparing a pair of graphs in graph product space, which is a non-negligible cost, especially for large-scale graphs. To resolve this difficulty, the use of the shortest path kernel \cite{borgwardt2005shortest} has been proposed. That method specifically examines a smaller set of shortest paths rather than random walks. Subsequent work on Weisfeiler--Lehman (WL) Graph Kernel \cite{Shervashidze_JMLR_2011_s} has brought great success. {They proposed a Graph Kernel with a similarity metric based on Weisfeiler--Lehman test}. Although this method yields attractive performance, it still presents difficulties of breadth-first strategies and $\mathcal{R}$-convolution methods as described in the Introduction section. {To provide a more valid notion of similarity of decomposite graph structure, related work \cite{kriege2016valid} proposed an optimal assignment kernel variant from Weisfeiler--Lehman Graph Kernel, which is Weisfeiler--Lehman Optimal Assignment (WL-OA) kernel. This kernel is based on an optimal bijection between parts of two graphs, which brings a better and more robust performance compared to the WL kernel.}

The second class includes Graph Kernels combined with optimal transport theory. Recent research \cite{togninalli2019Wasserstein} presents the Wasserstein-based Weisfeiler--Lehman Graph Kernel (WWL), which maps node embedding of a Weisfeiler--Lehman pattern to feature space and which computes kernel values using the Wasserstein distance of two point clouds in the feature space. They received better results than those yielded by the original Weisfeiler--Lehman kernel. GOT \cite{maretic2019got} uses optimal transport differently to compute the Wasserstein distance between two normal distributions derived by graph Laplacian matrices instead of generating walks or comparing vertex neighbors in graphs. 

Another report of recent work \cite{titouan2019optimal} raises difficulties that both the Wasserstein distance and the Gromov--Wasserstein distance are unable to accommodate the graph structure and feature information. To resolve this difficulty, they propose a notion of Fused Gromov--Wasserstein (FGW) distance, which considers both structure characteristics and feature information of two graphs. {To reduce computation costs for Wasserstein-distance-based Graph Kernel and to generate a graph embedding based on Wasserstein distance, related work \cite{kolouri2020wasserstein} proposes a new method that successfully embeds graphs into Euclidean space to compute approximate Wasserstein distance.}

{Aside from Graph Kernels, Graph Neural Networks (GNN) have also become a very hot topic in graph classification, node (link) prediction, and so on. A representative method of GNN in recent years is the Graph Convolution Network (GCN) \cite{kipf2016semi}, which is inspired by Convolution Neural Network (CNN) in computer vision. Another representative GNN is the Graph Isomorphism Network (GIN) \cite{xu2018powerful}. The GIN is based on a similar idea to Graph Kernels which specifically examines graph isomorphism. They prove the SUM aggregation to be a better aggregation scheme in message-passing of graphs and apply it in GIN, which has powerful performance in graph classification.}

\section{Longest Common Subsequence (LCS) Kernel} 
This section presents a metric between two labeled graphs through comparison of all the shortest path sequences. We will also present a fast kernel in the next section, which has drastically reduced computation. Therefore, we specifically designate this basic method in this section and its fast variant, respectively, as BLCS and FLCS.

\subsection{Basic concepts}
\label{Sec:BC}
This subsection first introduces several basic concepts in our LCS kernel. They show the operation of transforming the path into a manageable form of data, and present a formula of path sequence similarity used in graph comparison procedures.
\subsubsection{Shortest path serialization}
Given an undirected and connected graph $G = (\mathcal{V},\mathcal{E})$ with a set of vertices 
$\mathcal{V}=\left\{v_i\right\}_{i = 1}^{N}$ and a set of edges $\mathcal{E} = \left\{e_{ij}\right\}$, {where the $e_{ij}$ is equivalent to $e_{ji}$}, then both vertices and edges in $G$ are assigned a categorical label. To describe the subgraph structure of $G$, we chose the shortest path set which contains all shortest paths in $G$.

Let $\mathcal{P}$ denote the shortest path set of $G$, $\mathcal{P}$ defined as
\begin{equation*}
\setlength\abovedisplayskip{2pt}
\mathcal{P} \triangleq \left\{ p^\star_{i,j}| \forall v_i,v_j\in \mathcal{V}\right\},
\end{equation*}
where $p^\star_{i,j}$ represents the shortest path from vertex $v_i$ to $v_j$. Assuming that 
two vertices $v_k,v_l$ are in the shortest path $p^\star_{i,j}$, except for the start vertex $v_i$ and the 
end vertex $v_j$, then $p^\star_{i,j}$ can be expressed in sequence form:
{
\begin{eqnarray*}
\setlength\abovedisplayskip{3pt}
p^\star_{i,j}:(v_i,e_{i,k},v_k,e_{k,l},v_l,e_{l,j},v_j).
\end{eqnarray*}
}

Because of the difficulty in comparing paths directly, we perform a kind of {\it shortest path serialization} before path comparison, for which we serialize these paths as {\it label sequences}. We define the operation of this serialization as

\begin{definition}[Shortest Path Serialization]
\label{Def:SPS}
Let $l:\mathcal{V}\to \Sigma$ denote a function that maps a vertex object $v$ to its categorical 
node label assigned from a finite label alphabet $\Sigma$. Furthermore, let $w:\mathcal{E}\to \Sigma$ 
be the edge label mapping function.
{
Given a shortest path $p_{i,j}^\star$ from $v_i$ to $v_j$, and assuming that there are $N$ elements in $p_{i,j}^\star$ including vertices and edges, $p_{i,j}^\star(k)$ denotes the $k$-th element in $p_{i,j}^\star$.
The serialized shortest path sequence $\seq{x}$ is definable as
\begin{gather*}
	\seq{x}(k) = 
	\begin{cases}
		l(v_i), & k=1\cr
			\begin{cases}
		             l(p_{i,j}^\star(k)), & p_{i,j}^\star(k) {\rm\ is\ a\ vertex}\cr
		             -w(p_{i,j}^\star(k)), & p_{i,j}^\star(k) {\rm\ is\ an\ edge}
		        \end{cases},&1<k<N\cr
		l(v_j),& k=N.
	\end{cases}
\end{gather*}
In the special condition in which the graph has no edge label, $\seq{x}$ is
\begin{gather*}
	\seq{x}(k) =
	\begin{cases}
		l(v_i),&k=1\cr
		l(p_{i,j}^\star(2k-1)),&1<k<M\cr
		l(v_j),&k=M
	\end{cases}
\end{gather*}
where the length M of $\seq{x}$ is equal to $\frac{N+1}{2}$.
}
\end{definition}
It is noteworthy that we take the negative value of edge label as $-w(e_{i,j})$ so that edge labels are distinguished from node labels during path sequence comparison.

\subsubsection{Path sequence similarity}
Using shortest path serialization, we transform the path comparison problem to a sequence comparison problem. As described earlier, {to avoid the information loss problem, we use the LCS to compare path sequences, which compares paths with different lengths}
instead of simply judging whether or not these two paths are completely identical.

As sequence comparison strategies, many methods might be chosen, such as {\it longest common substring} (LCT) and {\it longest common subsequence} (LCS). The difference between them is that the {LCT} demands continuity of sequence above {LCS}. We prefer to use {LCS} rather than {LCT} because we discover {LCS} as typically more robust and stable than {LCT}. Assume three graphs with vertices labeled as graphs (a), (b) and (c) in Figure \ref{Fig.2}, where the vertices and edges in red denote one of the shortest paths in the graph. Generally, these three graphs should be regarded as mutually similar. However, the outputs of {LCS} and {LCT} differ considerably. Through shortest path serialization, these three shortest paths in graph (a), (b) and (c) are serialized as three path sequences $\seq{x}_a = (1,1,1,1,1)$, $\seq{x}_b = (1,1,1,2,1,1)$ and $\seq{x}_c = (1,1,1,1,1,1)$ respectively. The {comparison results} of {LCT} and {LCS} are shown in Table \ref{Tab.1}. It is apparent that, in the case of comparing $\seq{x}_a,\seq{x}_b$ and $\seq{x}_b,\seq{x}_c$, the length of {LCT} suddenly decreases because a point with the label of 2 suddenly shows up in $\seq{x}_b$. Consequently, {LCT} is more likely to be cut off by {mismatched} vertices.

\begin{figure}
\centering
\subfigure[]{
\includegraphics[width=0.3\textwidth]{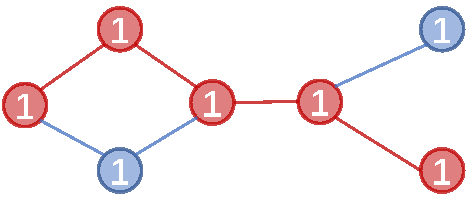} 
}
\subfigure[]{
\includegraphics[width=0.3\textwidth]{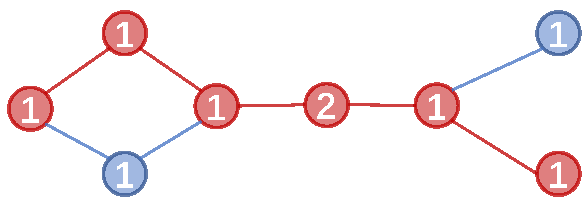} 
}
\subfigure[]{
\includegraphics[width=0.3\textwidth]{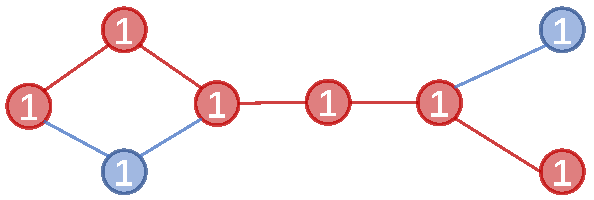} 
}
\caption{Three similar graphs with vertices labeled, where the vertices and edges in red denote one of the shortest paths in the graph. It is noteworthy that the difference between (b) and (c) is that (b) has a vertex with label ``2'' in the central position.} 
\label{Fig.2} 
\end{figure}

\begin{table}
\centering
\caption{{Longest common substring and subsequence of $\seq{x}_a$, $\seq{x}_b$ and $\seq{x}_c$. Each item in these tables shows comparison results between two among $\seq{x}_a$, $\seq{x}_b$ and $\seq{x}_c$ derived from examples in Figure \ref{Fig.2}. Table (a) shows them in the form of sequence. Table (b) shows their length.}} 
\label{Tab.1} 
\subfigure[{Longest common parts}]{
\begin{tabular}{cccc}
\hline
Methods&$\seq{x}_a,\seq{x}_b$&$\seq{x}_a,\seq{x}_c$&$\seq{x}_b,\seq{x}_c$\\
\hline
LCT&(1,1,1)&(1,1,1,1,1)&(1,1,1)\\
LCS&(1,1,1,1,1)&(1,1,1,1,1)&(1,1,1,1,1)\\
\hline
\end{tabular}
}

\subfigure[{Length of longest common parts}]{
\begin{tabular}{cccc}
\hline
Methods&$\seq{x}_a,\seq{x}_b$&$\seq{x}_a,\seq{x}_c$&$\seq{x}_b,\seq{x}_c$\\
\hline
LCT&3&5&3\\
LCS&5&5&5\\
\hline
\end{tabular}
}
\end{table}
Using the length of LCS, we propose our formula of path sequence similarity as
\begin{definition}[Path Sequence Similarity]
\label{Def:PSS}
Given two path sequences $\seq{x}_1,\seq{x}_2 \in \mathcal{X}$, their similarity $F_{\rm sim}: \mathcal{X} \times \mathcal{X} \to \mathbb{R}  $ is defined as
\begin{equation}
\label{Eq:PSS}
F_{\rm sim}(\seq{x}_1,\seq{x}_2):= \frac{F_{\rm lcs}(\seq{x}_1,\seq{x}_2)}{\max(|\seq{x}_1|,|\seq{x}_2|)},
\end{equation}
where {$F_{\rm lcs}:\mathcal{X}\times\mathcal{X}\to\mathbb{R}$} represents the function that returns the length of the Longest 
Common Subsequence of $\seq{x}_1,\seq{x}_2$.
\end{definition}
We use the maximum length of the objective path 
sequences as a denominator to limit its value in $[0,1]$, which prevents the Sinkhorn algorithm from diverging when computing optimal transport plans. {This is true because, if we have a cost $\mat{C}_{i,j}$ with an overly large value, the kernel $\mat{K}_{i,j} = e^{-\frac{\mathbf{C}_{i,j}}{\epsilon}}$ in the Sinkhorn algorithm will become too negligible to be stored in memory as positive numbers, which causes endless iterations, as described in Remark 4.7 in Peyr{\'e}'s work \cite{Peyre_2019_OTBook_s}.} When two path sequences have a 
longer common subsequence, the value of their similarity will be larger. From the perspective of the
graph, a large similarity of path sequences shows a large similarity of subgraph structures.

\subsection{BLCS Kernel in graph comparison}
{The graph comparison of BLCS kernel is divided into the following steps.}

\subsubsection{Step 1: Graph representation generation}
Given two undirected and connected graphs $G_1(\mathcal{V}_1,\mathcal{E}_1)$ and $G_2(\mathcal{V}_2,\mathcal{E}_2)$ with nodes and edges labeled, then using a classical algorithm like Floyd--Warshall or the Dijkstra algorithm,
the shortest path sets $\mathcal{P}_1$ and $\mathcal{P}_2$ of $G_1$ and $G_2$ are obtainable, respectively.
Through shortest path serialization, we first transform all paths in $\mathcal{P}_1$ and $\mathcal{P}_2$ as shown below. 
\begin{equation*}
\setlength\abovedisplayskip{3pt}
\begin{aligned}
	\mathcal{X}_1=\left\{ (\seq{x}_1)_{i,j} |(\seq{x}_1)_{i,j} = F_{\rm serialize}(p_{i,j}^\star),p_{i,j}^\star\in \mathcal{P}_1\right\}\cr
	\mathcal{X}_2=\left\{ (\seq{x}_2)_{i,j} |(\seq{x}_2)_{i,j} = F_{\rm serialize}(p_{i,j}^\star),p_{i,j}^\star\in \mathcal{P}_2\right\}	
\end{aligned}
\end{equation*}
Therein, $F_{\rm serialize}$ denotes the operation of the shortest path serialization described earlier. $\mathcal{X}_1$ and $\mathcal{X}_2$ respectively represent two path sequence sets derived from $\mathcal{P}_1$ and $\mathcal{P}_2$, respectively describing the subgraph structures of $G_1$ and $G_2$. For every path sequence $(\seq{x}_1)_{i,j}$ and $(\seq{x}_2)_{i,j}$, the superscript denotes the graph to which they belong. {In a special case in which there is not only one shortest path existing for two certain vertices, we use only one of them randomly. Specifically, we use the Floyd--Warshall algorithm to extract the shortest path sequences, which are based on dynamic programming. In this situation, the extracted shortest path of two vertices will be the first found during dynamic programming, although other shortest paths might exist.}

It is noteworthy that some path sequences are exactly the same as those in the path sequence set. To avoid an increase of identical sequences, we do not add them to the sequence set $\mathcal{X}$ and store the number of identical sequences in a mass vector $\vec{m}$. For example, the $i$-th element in $\vec{m}$ denotes the number of sequences that are identical to the $i$-th path sequence in $\mathcal{X}$. In doing so, we obtain two mass vectors $\vec{m}_1$ and $\vec{m}_2$, respectively belonging to $\mathcal{X}_1$ and $\mathcal{X}_2$. 

\subsubsection{Step 2: Wasserstein Distance Computation in LCS Metric Space}
To compare the path sequence sets $\mathcal{X}_1$ and $\mathcal{X}_2$, we propose a metric space of path sequence over the union of $\mathcal{X}_1$ and $\mathcal{X}_2$, where we use the path sequence similarity to define the metric.
\begin{definition}[LCS Metric Space]
\label{Def:SSS}
Given two path sequence sets $\mathcal{X}_1,\mathcal{X}_2$, and $\mathcal{X} := \mathcal{X}_1 \cup \mathcal{X}_2$ is their union. The LCS metric space of $\mathcal{X}_1$ and $\mathcal{X}_2$ is written as $S(\mathcal{X},d)$, where the metric $d:\mathcal{X} \times \mathcal{X} \to \mathbbm{R}$ is defined as presented below.
\begin{equation}
\setlength\abovedisplayskip{2pt}
\label{Eq:LCSmetric}
	d(\seq{x}_1,\seq{x}_2) = 1- F_{\rm sim}(\seq{x}_1,\seq{x}_2),
\end{equation}
where $\seq{x}_1\in \mathcal{X}_1,\seq{x}_2 \in \mathcal{X}_2$.
\end{definition}

\begin{theorem}
\label{Theo:metric}
The metric function $d:\mathcal{X} \times \mathcal{X} \to \mathbb{R}$ of LCS metric space is a metric.
\end{theorem}
$\mathcal{X}_1$ and $\mathcal{X}_2$ can be regarded as two distributions in their LCS metric space, where the path sequence is a discrete point of these distributions. Finally, we propose to leverage the Wasserstein distance between these two distributions to define our graph distance as

\begin{definition}[LCS Graph Distance]
Given two undirected and connected graphs $G_1(\mathcal{V}_1,\mathcal{E}_1), G_2(\mathcal{V}_2,\mathcal{E}_2) \in \mathcal{G}$ and their respective path sequence sets $\mathcal{X}_1$ and $\mathcal{X}_2$ and the mass vectors $\vec{m}_1$ and $\vec{m}_2$, then the LCS graph distance $d_G: \mathcal{G} \times \mathcal{G}\to \mathbb{R}$ between $G_1$ and $G_2$ is defined as
\begin{equation*}
\label{Eq.LCSGD}
d_{G}(G_1,G_2) = \mathcal{W}_1(\mathcal{X}_1,\mathcal{X}_2) = \left\langle \mat{T}^\star\left(\mat{D},\frac{\vec{m}_1}{\|\vec{m}_1\|_1},\frac{\vec{m}_2}{\|\vec{m}_2\|_1}\right),\mat{D} \right\rangle
\end{equation*}
where $\mathcal{W}_1(\mathcal{X}_1,\mathcal{X}_2)$ represents the 1-Wasserstein distance between distribution $\mathcal{X}_1$ and $\mathcal{X}_2$, and where $\mat{D}$ denotes the ground distance matrix including distances $d(\seq{x}_1,\seq{x}_2)$ between $\seq{x}_1 \in \mathcal{X}_1$ and $\seq{x}_2 \in \mathcal{X}_2$ in the LCS metric space $S(\mathcal{X}_1\cup \mathcal{X}_2,d)$. Here, $\mat{T}^\star\left(\mat{D},\frac{\vec{m}_1}{\|\vec{m}_1\|_1},\frac{\vec{m}_2}{\|\vec{m}_2\|_1}\right)$ denotes the optimal transport plan w.r.t. mass vectors $\vec{m}_1$ and $\vec{m}_2$ of $\mathcal{X}_1$ and $\mathcal{X}_2$, respectively, and where $\mat{D}$, where $\|\cdot\|_1$ denotes the $l_1$ norm.
\end{definition}

We use the LCS graph distance to compute a similarity measure between two graphs. The measure will be used as a kernel value in the machine learning algorithm. Here we combine the LCS graph distance with the Laplacian kernel function to construct a Graph Kernel. The LCS Graph Kernel is then defined as presented below. The proof of Theorem \ref{Theo:metric} is provided in Supplementary Materials.

\begin{definition}[LCS Graph Kernel]
Given a set of graphs $\mathcal{G}$, then $G_i$ and $G_j$ are two arbitrary graphs in $\mathcal{G}$. The LCS Graph Kernel $k:\mathcal{G}\times\mathcal{G}\to\mathbb{R}$ is defined as
\begin{equation}
\label{Eq:LCSK}
k_{\rm LCS}(G_i, G_j)=e^{-\lambda d_G(G_i,G_j)},
\end{equation}
where $d_G$ denotes the LCS graph distance, and where $\lambda$ is within the range of $(0,+\infty]$.
\end{definition}

{Aside from LCS metric, one can readily replace $d:\mathcal{X}\times\mathcal{X}\to \mathbb{R}$ with another string (sequence) metric such as Hamming distance or Levenshtein distance to construct a new metric space. We provide a variant kernel by applying the Levenshtein distance and comparing it with our proposed LCS Graph Kernel in numerical experiments. One can refer to Section ``Variant Graph Kernels from LCS Kernel'' and ``Numerical Experiments'' for details.}

\subsection{Time complexity}
One can discuss the time complexity of the BLCS kernel by dividing the algorithm into two parts. In the first part, the algorithm first generates a path sequence set for graph representation. In the BLCS kernel, we use all shortest paths in an undirected and connected graph. Assuming that the objective graph has $N$ vertices, then the time complexity for graph representation generation is $O(N^3)$ for each graph {if we use the Floyd--Warshall algorithm for shortest path searching}. Furthermore, $N^2$ path sequences are obtained after graph representation generation.
 In the second part, the algorithm performs a one-to-one comparison of path sequences of two graphs to compute the ground distance matrix for calculating the Wasserstein distance. {Using dynamic programming, one iteration of LCS comparison of two sequences with respective lengths $n$ and $m$ costs computation in $O(n \times m)$ \cite{wagner1974string}. Therefore, the computational cost of each comparison is less than $O(l^2)$, where $l$ denotes the longest length $l$ among all path sequences}. Under this situation, performing a one-to-one comparison will require computation in $O(N^4l^2)$.
 
\subsection{Kernel validity}
{It is noteworthy that the positive definiteness of the LCS Graph Kernel is not proven currently.}
The literature shows that only kernels using Wasserstein distance with discrete metrics have been proven as valid kernels \cite{gardner2015earth}, which is unsuitable for our kernel function. {However, many studies have examined machine learning based on indefinite kernels such as Kre{\u\i}n SVM \cite{loosli2015learning}, which shows that some indefinite kernel functions can still be applied to machine learning, where the indefinite parts are treated as noise. In addition, our experiments using small datasets and large datasets have demonstrated that the indefinite parts have no great effect on performance. For verification, we provide an experiment to demonstrate the feasibility of our LCS Graph Kernel.} Specifically, we follow the idea of Kre{\u\i}n SVM, which first removes all negative eigenvalues of the Gramian matrix and then reverts it to obtain a new Gramian matrix with no negative eigenvalue. We perform the same operation on our BLCS and FLCS kernel, and then compare the performances of the original Gramian matrix and the new one from which negative eigenvalues have been removed. For details of experimentation and results, please refer to the numerical evaluation section.

\section{Fast LCS Kernel}
We present a fast LCS (FLCS) Graph Kernel based on adjacent point merging in the LCS metric space. It is noteworthy that this greatly reduces computational costs while preserving (sometimes improving) the BLCS kernel accuracy.

\subsection{Improving strategies}
Through comparison of shortest path sequences in a single graph, we discovered that many similar but not completely identical shortest path sequence samples are repeated. This repetition might lead to unnecessarily redundant calculations for comparison. Moreover, the finally obtained ground distance matrix might exacerbate the degradation of classification performances. To alleviate these issues in the BLCS kernel, we propose two strategies: fragmented path sequence removal and adjacent point merging, and specific examination of reduction of the path sequence set size.

\subsubsection{Fragmented path sequence removal}
The fragmented path sequences are those sequences having negligible length, which are short because they carry only limited information. Dealing with these path sequences might not only cause unnecessary computation; it might also lead to negative effects on Wasserstein distance calculation. Therefore, we perform fragmented path sequence removal as preprocessing.

Given a path sequence set $\mathcal{X}$ with $N^2$ shortest path sequences included, we {modify} $\mathcal{X}$ with length-limited conditions as
\begin{equation*}
\setlength\abovedisplayskip{2pt}
\label{Eq.LenR}
	\mathcal{X}_{\rm new} = \left\{\seq{x}_i|\seq{x}_i \in \mathcal{X}, |\seq{x}_i| \geq \rho \cdot L_{\rm max}\right\},
\end{equation*}
where $L_{\rm max}$ denotes the length of the longest one among all path sequences in $\mathcal{X}$, and where $\rho$ is the parameter of the removal ratio, valued within as $(0,1)$. When $\rho=0$, all of 
the path sequences in original $\mathcal{X}$ will be reserved. When $\rho=1$, only the path sequences having the longest length will be reserved.

\subsubsection{Adjacent point merging in LCS metric space}
As described above, path sequences in the original $\mathcal{X}$ are redundant to a certain degree. We can observe this point in the LCS metric space as Figure \ref{Fig.3a}, where points in distribution are usually numerous and dense. Actually, when we map an original shortest path sequence set of a graph to the LCS metric space, these points in the distribution are usually close to a local center. They are divisible into several parts.

\begin{figure}
\setlength{\belowcaptionskip}{-0.3cm} 
\centering
\subfigure[Before merging]{
\label{Fig.3a}
\includegraphics[width=0.3\textwidth]{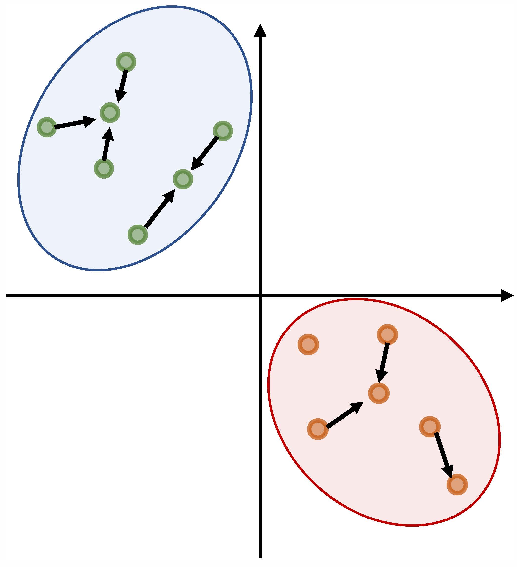} 
}
\subfigure[After merging]{
\label{Fig.3b}
\includegraphics[width=0.3\textwidth]{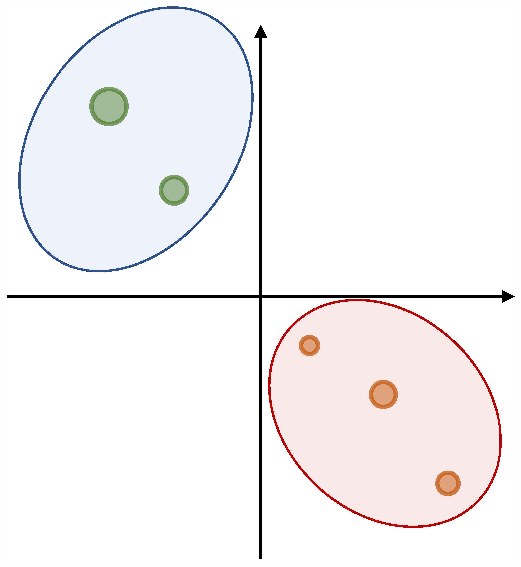} 
}
\caption{Visual summary of adjacent point merging, where blue and red areas respectively denote two distributions in LCS metric space and where the point size shows their mass: (a) two distributions before point merging, arrows indicate the merging direction; and (b) the same distributions after merging, where the adjacent points are combined, as are their masses.}  
\label{Fig.3} 
\end{figure}

Consequently, with the aim of reducing the density of points in metric space, {we can apply some clustering methods or point cloud sampling methods. Several methods have been tested, including the $k$-medoids algorithm \cite{maranzana1963location, park2009simple}, agglomerative clustering algorithm \cite{rokach2005clustering}, and farthest point sampling \cite{moenning2003fast}. For clustering methods, we found that they have a high computational cost for optimizing the solution, which is not worth the loss. For point cloud sampling method such as farthest point sampling, it has low computational cost compared to clustering methods. However, we found that during its sampling process, the information of many unselected points is abandoned. Their mass is lost. Actually, the Wasserstein distance considers the transportation of mass of discrete points in distribution, which means that the mass of unselected points can be used in this situation. Therefore,} we propose a {mass-transfer} operation of adjacent point merging. For this purpose, we follow the same strategy of recalculation of the mass vector as that used for BLCS. As shown in Figure \ref{Fig.3}, points in distributions are merged to their local center. Masses of merged points will also be accumulated to the local center. They appear to grow larger in Figure \ref{Fig.3b}.

\begin{definition}[Adjacent Point Merging]
\label{Def:APM}
Given a discrete distribution $\mathcal{X}$ in LCS metric space $S(\mathcal{X},d)$, with parameter $s$ as the merging radius, then each point in $\mathcal{X}$ is assigned an equivalent mass $1$. Subsequently, we define a set $\mathcal{X}^\prime$ to store local center points and to add a random point in it. For the remainder of points in $\mathcal{X}$, we define the merging attempt $F_{\rm attempt}:\mathcal{X}\to\{0,1\}$ as shown below.
\begin{equation*}
\setlength\abovedisplayskip{2pt}
F_{\rm attempt}(\seq{x}) =
\begin{cases}
1,&\min\limits_{\scriptsize \seq{x}^\prime \in \mathcal{X}^\prime} d(\seq{x},\seq{x}^\prime) \leq s\cr
0,&\min\limits_{\scriptsize \seq{x}^\prime \in \mathcal{X}^\prime} d(\seq{x},\seq{x}^\prime) > s.
\end{cases}
\end{equation*}
We perform merging attempts on points in the original distribution. If we obtain $F_{\rm attempt}(\seq{x}) = 1$, then $\seq{x}$ will be merged with the local center $\seq{x}^\prime = \argmin\limits_{\scriptsize \seq{x}^\prime \in \mathcal{X}^\prime}d(\seq{x},\seq{x}^\prime)$, with their masses combined as $m^\prime := m^\prime + m$, where $m^\prime,m$ denotes the mass of $\seq{x}^\prime,\seq{x}$. If we get $F_{\rm attempt}(\seq{x}) = 0$, then $\seq{x}$ will be regarded as a new local center and added to $\mathcal{X}^\prime$. Eventually, a merged distribution $\mathcal{X}^\prime$ with a mass vector $\vec{m}$ of all remaining points is obtained.
\end{definition}

As a result of adjacent point merging, Figure \ref{Fig.3b} shows that the original dense distributions are converted to new distributions including fewer points with combined mass. After obtaining $\mathcal{X}$ and $\vec{m}$, the FLCS kernel follows the same step of the BLCS kernel to compute kernel values.

\subsection{Time complexity}
Similarly to the BLCS kernel, we discuss the FLCS kernel time complexity under the same assumption. In the first part, for adjacent point merging, we consider the worst situation with $\frac{N(N+1)}{2}$ times of comparison. An adjacent point combination will cost a maximum computational cost in $O(N^2l^2)$. In the second part, assuming that $K$ path sequences remain after adjacent point merging, the path sequence comparison of two graphs will cost $O(K^2l^2)$. The time complexity for the FLCS kernel is $O(N^2l^2 + K^2l^2)$. This time complexity is reduced dramatically by comparison to $O(N^4)$ in the BLCS kernel. We also conduct numerical evaluations to compare elapsed times with other state-of-the-art methods, as presented in the next section.

\section{Variant Graph Kernels from LCS Kernel}
{This section presents some variant Graph Kernels from our proposed LCS Graph Kernel for ablation studies. These variant Graph Kernels include implementations of LCS kernels based on the $\mathcal{R}$-convolution framework, or with a different sequence (or string) comparing algorithm, or computing the optimal transport plan based on a metric of the sequence length difference. Our numerical evaluation is expected to compare the performance of these variant Graph Kernels with our proposed LCS kernel.}

\subsection{FLCS kernel based on the $\mathcal{R}$-convolution framework}
To show improvement of OT framework (Wasserstein distance) {compared} to the $\mathcal{R}$-convolution framework, we implement our FLCS kernel based on the $\mathcal{R}$-convolution framework, which we call FLCS-R Graph Kernel. The FLCS-R Graph Kernel on two graphs $G_1, G_2$ is computed as
\begin{eqnarray*}
	k_{\rm FLCS-R}(G_1, G_2) = \sum_{\seq{x}_1\in\mathcal{X}_1, \seq{x}_2\in\mathcal{X}_2} F_{\rm sim}(\seq{x}_1, \seq{x}_2),
\end{eqnarray*}
where $\mathcal{X}_1, \mathcal{X}_2$ is respective path sequence sets of $G_1,G_2$, and $F_{\rm sim}$ is the path sequence similarity in Eq. (\ref{Eq:PSS}). 

\subsection{Levenshtein kernel}
As described in the earlier section, the graph comparison problem can be transformed into a sequence (or string) comparison problem. However, there are many choices {for sequence (or string) metrics}. To show differences between LCS and other algorithms, we implement a Levenshtein Graph Kernel which uses sequence metrics based on the Levenshtein distance \cite{levenshtein1966binary}. The Levenshtein Graph Kernel shares the same framework as FLCS Graph Kernel, which also uses the OT framework and {adjacent point merging operation}. For the Levenshtein kernel, we construct the metric space using the following metric
\begin{eqnarray*}
	d_{\rm L}(\seq{x}_1, \seq{x}_2) = \frac{F_{\rm Levenshtein}(\seq{x}_1, \seq{x}_2)}{\max(|\seq{x}_1|, |\seq{x}_2|)},
\end{eqnarray*}
where $\seq{x}_1, seq{x}_2$ denotes two sequences, and where $F_{\rm Levenshtein}$ is the function of computing the Levenshtein distance of two sequences.

\subsection{FLCS kernel with no label}
Under some extreme conditions in which the graphs have no vertex label or edge label or have only one label, the comparison of LCS will degenerate into the comparison of sequence lengths. Considering these conditions, we implement a version of the FLCS kernel which only compares the lengths of sequences. This implementation is similar to the shortest path kernel with adjacent point merging and the OT framework. We construct the metric space for this no-label version using the following metric
\begin{eqnarray*}
	d_{\rm Len}(\seq{x}_1, \seq{x}_2) = \frac{||\seq{x}_1| - |\seq{x}_2||}{\max(|\seq{x}_1|, |\seq{x}_2|)},
\end{eqnarray*}
where $\seq{x}_1, \seq{x}_2$ denotes two sequences, the absolute value of their difference is divided by the length of the longest among them.

\section{Numerical Evaluation}
This section compares our method with other state-of-the-art graph classification methods by the performance of numerical experiments. Particularly, this evaluation reveals that the proposed BLCS and FLCS kernels outperform many state-of-the-art methods in the most widely used datasets.

\subsection{Datasets}
For real-world datasets, we use several widely used benchmark datasets. Among them, MUTAG \cite{debnath1991structure}, PTC-FM, PTC-MR \cite{helma2001predictive}, AIDS \cite{zaharevitz2015aids} and {UACC257 \cite{yan2008mining}} include graphs with discrete labels of both vertices and edges. Also, MSRC-9, MSRC-21, MSRC-21C \cite{neumann2016propagation} include graphs with discrete vertex labels. The BZR, COX2 \cite{sutherland2003spline}, ENZYMES \cite{schomburg2004brenda}, and PROTEINS \cite{borgwardt2005protein} datasets include graphs with discrete and continuous vertex attributes. In addition, BZR-MD, COX2-MD, and ER-MD \cite{sutherland2003spline} include graphs with discrete labels of vertices and with both discrete and continuous attributes of edges. The IMDB-B and IMDB-M \cite{yanardag2015deep} datasets include both unlabeled and unattributed graphs. The datasets above are available in Tudataset \cite{Morris+2020}. Also, ogbg-molhiv is a large-scale dataset from the Open Graph Benchmark \cite{hu2020open}, including graphs with both discrete vertex and edge labels; each vertex has a nine-dimensional node feature.

\begin{table}[H]
\centering
\caption{Statistical information of datasets}
\label{Tab:data1}
\begin{tabular}{lccccc}
\toprule
Dataset&Category&Graphs&Classes&Avg. $|\mathcal{V}|$&Avg. $|\mathcal{E}|$\\
\midrule
MUTAG&bio&188&2&17.93&19.79\cr
PTC-FM&bio&349&2&14.11&14.48\cr
PTC-MR&bio&344&2&14.29&14.69\cr
AIDS&bio&2000&2&15.69&16.20\cr
BZR&bio&405&2&35.75&38.36\cr
COX2&bio&467&2&41.22&43.45\cr
ENZYMES&bio&600&6&32.63&62.14\cr
PROTEINS&bio&1113&2&39.06&72.82\cr
BZR-MD&bio&405&2&21.30&225.06\cr
COX2-MD&bio&303&2&26.28&335.12\cr
ER-MD&bio&446&2&21.33&234.85\cr
UACC257&bio&39988&2&26.09&28.12\cr
ogbg-molhiv&bio&41127&2&25.50&27.50\cr
\midrule
MSRC-9&cv&221&8&40.58&97.94\cr
MSRC-21&cv&563&20&77.52&198.32\cr
MSRC-21C&cv&209&20&40.28&96.60\cr
\midrule
IMDB-B&soc&1000&2&19.77&96.53\cr
IMDB-M&soc&1500&3&13.00&65.94\cr
\bottomrule
\end{tabular}
\end{table}

Table \ref{Tab:data1} shows specific information of datasets that we use. We divide datasets into three categories according to their sources as {\it bio}, {\it cv}, and {\it soc}, respectively denoting bioinformatics, computer vision, and social network. The Avg.$|\mathcal{V}|$ and Avg.$|\mathcal{E}|$ columns respectively show the average number of vertices and edges in a dataset, where it is apparent that the datasets BZR-MD, COX2-MD, and ER-MD are dense graph datasets.

Because our proposed method only supports graphs with a {single} discrete label of vertices and/or edges, we apply some preprocessing before experiments: (1) We remove continuous attributes from original graphs, making sure that only discrete labels remain. (2) For those unlabeled graphs, we use the degree measures of the respective vertices as attributes. {(3) For those graphs with discrete multi-label or multi-dimensional features, we apply a perfect hash to map this feature to a single discrete label.}

\subsection{Experiment setup}
For numerical experiments, we test both BLCS and FLCS kernels to evaluate our improved strategies. {Moreover, we test all variant Graph Kernels from LCS kernel for ablation studies, which includes the FLCS-R Graph Kernel, Levenshtein Graph Kernel (LevenshteinK), and FLCS kernel with no label (FLCS-Len).}  For the FLCS kernel and its variant kernels, we set up a grid search to find the best pair of parameters, where the removal ratio $\rho$ is within {$\{0, 0.2\}$} and the merging radius $s$ is within {$\{0.2, 0.5\}$}.

We compare results obtained using our methods with those from other methods of {four} categories: (1) traditional Graph Kernel, (2) OT-based Graph Kernel, (3) graph-embedding method, and {(4) Graph Neural Network (GNN)}. In the first category, we chose the Weisfeiler--Lehman Kernel (WLK), Weisfeiler--Lehman Optimal Assignment Kernel (WL-OA), and Shortest Path Kernel (SPK), which we described in the related work section. For the implementation of all the traditional Graph Kernels, we use the Grakel python library \cite{siglidis2020grakel}. For the second category, we chose the Wasserstein--Weisfeiler--Lehman Kernel (WWL) and Fused Gromov--Wasserstein Kernel (FGW). For the third category, we chose the Anonymous Walk Embeddings (AWE) \cite{ivanov2018anonymous}, which is a graph-embedding method based on anonymous walks in a graph. {For the last category, we chose the Graph Isomorphism Network (GIN) \cite{xu2018powerful}.}
{Parameter $H$ in WLK, WL-OA and WWL is adjusted using grid-search, {where we seek} the best one from 1 to 10}. The shortest path matrix {and Hamming distance measure} are used in FGW. {Moreover, the parameter alpha of FGW is adjusted using grid-search on $\{0.2,0.5,0.8\}$. For the shortest path kernel, we use the non-label implementation for unattributed graphs, and with label implementation for the attributed dataset, which is based on the Grakel library. For AWE, we apply grid-search to the parameter step within $\{2,3,4\}$. For GIN, we use the best pair of the number of hidden layers and hidden dimension, respectively, within $\{1,2,3,4,5\}$ and $\{32,64,128\}$.} For all kernel methods based on a Laplacian kernel, we use the best parameter $\lambda$ among \{0.0001, 0.001, 0.01, 0.1, 1, 10\}.

 For classification, we train a multi-class LibSVM classifier using a one-vs.-one approach and apply $10$ runs of $10$-fold {nested} cross-validation \cite{stone1974cross} {(using an outer cross-validation for evaluation and an inner cross-validation for hyperparameter adjustment)}, which is $100$ runs of classification in all. For parameter adjustment of SVM, we apply a grid-search with SVM parameter $C$ within $\{ 0.001, 0.01, 0.1, 1, 10, 100, 1000\}$ for every cross-validation. Then we calculate an average accuracy and a standard deviation after classification. {For the classification of a large-scale dataset, to reduce computation and memory costs, we use the EnsembleSVM \cite{JMLR:v15:claesen14a} as a classifier, which is an implementation of LibSVM based on ensemble learning. This classifier will train multiple SVMs on small subsets of the original dataset; then it will use these weak classifiers to construct a strong classifier. EnsembleSVM helps divide the Gramian matrix into several small Gramian matrices, which generates less computation and memory cost. For classification of GNN, we also apply 10 runs of ten-fold cross-validation, after adjusting the best pair of parameters on an extra pair of 10:1 train-test sample sets. For the computation of average accuracy and a standard deviation, we use the same way as the evaluation in \cite{xu2018powerful}, which is presented in the github source code of their work. We first train the GNN for 200 epochs and evaluate each epoch using a test set, from which one can obtain 10 accuracy curves for 10 folds. Furthermore, we choose the one among all epochs as our final epoch, which has the best average accuracy of these 10 folds. Finally, we use the average accuracy and standard deviation of this epoch as our result. } For this study, we then run these experiments on a computer (16 GB RAM, 2.60 GHz Core i7 CPU; Intel Corp. and GeForce RTX 2060; Nvidia Corp.).

\subsection{Classification of small datasets}
We perform classification experiments on {small} datasets belonging to different categories. Tables \ref{Tab:acc1a}, \ref{Tab:acc1b}, \ref{Tab:acc1c}, and \ref{Tab:acc2} show the average classification accuracies on graphs, where the top {three} accuracies of each dataset are in bold. It is noteworthy that the results marked with {``Out of Time'' indicate the elapsed time as longer than 10 days.}

Overall, our proposed methods {including LCS kernels and their variant kernels} are shown to outperform many state-of-the-art methods in some datasets. They also keep good performance on most datasets. As shown in Tables \ref{Tab:acc1a}, \ref{Tab:acc1b} and \ref{Tab:acc1c}, in normal datasets such as {COX2, MSRC-9, IMDB-M, ENZYMES, BZR-MD and COX2-MD}, our proposed methods perform the best. We also find that in some datasets of the computer vision category (MSRC-9, MSRC-21, MSRC-21C), our proposed methods outperform many other well-known Weisfeiler--Lehman-based methods such as WL, WL-OA and WWL. {For LCS kernels and their variant kernels, it is apparent that some variant kernels perform the best on some of the datasets, such as FLCS-R on MSRC-9, MSRC-21, MSRC-21C and BZR-MD, and LevenshteinK on PROTEINS, IMDB-M, COX2-MD and ER-MD, but FLCS maintains good average accuracies that are not much different from those obtained using LevenshteinK. Also, for other datasets such as MUTAG, PTC-FM, PTC-MR, results show that FLCS outperforms LevenshteinK. Therefore, FLCS has the most stable performance on each dataset, which demonstrates its robustness compared to other variant kernels. Comparison of the results obtained from FLCS and FLCS-R reveals that the OT framework improves robustness and performance from the $\mathcal{R}$-convolution framework. Moreover, because we use the with-label implementation of SPK, it outperforms FLCS-Len on most attributed datasets. However, for unattributed datasets such as IMDB-B and IMDB-M, results show that FLCS-Len, which can be regarded as an optimal transport implementation of SPK, has better performance than that of SPK.}

We also perform dense graph experiments that test our proposed methods on a dataset with dense graphs, as shown in Table \ref{Tab:acc2}, where our proposed methods also retain good performance. Furthermore, elapsed time experiments are presented in Table \ref{Tab:tim}, where we use $\rho = 0.2, s = 0.5$ for the FLCS kernel. 
From Table \ref{Tab:tim}, it is apparent that the FLCS kernel has greatly reduced the computational cost from the BLCS kernel, also that it is competitive in computational speed among OT-based methods.

\begin{table}
\centering
\caption{Average classification accuracy on graphs with discrete vertex and edge labels}
\label{Tab:acc1a}
\begin{tabular}{lcccc}
\toprule
\!\!M{\scriptsize ETHOD}\!\!&MUTAG&PTC-FM&PTC-MR&AIDS\cr
\midrule
WLK&79.31$\pm$7.90&62.78$\pm$7.67&62.23$\pm$8.09&97.90$\pm$0.95\cr
SPK&83.24$\pm$8.80&63.35$\pm$5.92&59.33$\pm$8.78&\CBLG{99.58$\pm$0.43}\cr
WL-OA&82.72$\pm$7.09&\CBLG{64.92$\pm$6.30}&\CBLG{63.45$\pm$8.63}&99.40$\pm$0.61\cr
\midrule
WWL&\CBLG{85.90$\pm$7.39}&\CBLG{66.59$\pm$7.18}&\CBLG{65.31$\pm$7.06}&98.23$\pm$0.90\cr
FGW&79.20$\pm$8.44&64.81$\pm$5.19&56.57$\pm$6.83&98.37$\pm$0.82\cr
\midrule
AWE&74.68$\pm$9.07&\CBLG{65.03$\pm$5.23}&56.36$\pm$5.01&80.00$\pm$0.00\cr
\midrule
GIN&\CBLG{88.59$\pm$6.89}&63.69$\pm$7.41&\CBLG{64.76$\pm$7.67}&97.87$\pm$1.08\cr
\midrule
BLCS&85.54$\pm$8.21&61.63$\pm$5.79&59.70$\pm$7.03&Out of Time\cr
FLCS&\CBLG{85.80$\pm$6.54}&64.32$\pm$6.24&62.49$\pm$7.36&\CBLG{99.38$\pm$0.47}\cr
FLCS-R&78.89$\pm$8.26&63.68$\pm$7.10&58.02$\pm$6.89&99.02$\pm$0.63\cr
LevenshteinK&83.81$\pm$8.19&62.63$\pm$5.61&59.27$\pm$7.47&\CBLG{99.36$\pm$0.55}\cr
FLCS-Len&77.01$\pm$8.50&60.60$\pm$6.16&55.04$\pm$7.75&97.66$\pm$1.14\cr
\bottomrule
\end{tabular}
\vspace*{0.3cm}

\caption{Average classification accuracy on graphs with only a discrete vertex label}
\label{Tab:acc1b}
\begin{tabular}{lccccc}
\toprule
\!\!M{\scriptsize ETHOD}\!\!&BZR&COX2&PROTEINS&MSRC-9&MSRC-21\cr
\midrule
WLK&\CBLG{88.58$\pm$4.02}&\CBLG{81.76$\pm$4.28}&74.24$\pm$3.75&90.10$\pm$4.79&88.45$\pm$3.46\cr
SPK&86.53$\pm$4.55&81.90$\pm$5.12&\CBLG{76.02$\pm$3.80}&91.64$\pm$5.53&90.47$\pm$3.55\cr
WL-OA&\CBLG{87.98$\pm$3.67}&80.47$\pm$4.44&74.24$\pm$3.75&90.60$\pm$4.94&90.99$\pm$3.31\cr
\midrule
WWL&\CBLG{88.37$\pm$3.84}&81.75$\pm$3.71&74.13$\pm$3.47&90.69$\pm$4.57&89.75$\pm$3.52\cr
FGW&78.77$\pm$1.01&78.15$\pm$0.80&59.56$\pm$0.16&13.57$\pm$0.17&6.04$\pm$0.87\cr
\midrule
AWE&78.77$\pm$1.01&78.15$\pm$0.80&66.59$\pm$2.85&46.44$\pm$8.14&20.98$\pm$4.47\cr
\midrule
GIN&87.76$\pm$4.73&\CBLG{82.59$\pm$4.42}&73.72$\pm$4.27&\CBLG{93.26$\pm$4.00}&\CBLG{93.51$\pm$2.70}\cr
\midrule
BLCS&86.15$\pm$4.18&\CBLG{83.53$\pm$4.85}&Out of Time&91.23$\pm$4.71&92.69$\pm$3.17\cr
FLCS&86.52$\pm$4.11&79.95$\pm$3.87&\CBLG{74.88$\pm$3.72}&91.23$\pm$4.71&\CBLG{92.94$\pm$3.09}\cr
FLCS-R&82.32$\pm$3.91&81.08$\pm$4.23&74.45$\pm$4.05&\CBLG{93.57$\pm$4.20}&\CBLG{93.12$\pm$2.68}\cr
LevenshteinK&87.29$\pm$4.54&81.64$\pm$4.44&\CBLG{75.37$\pm$3.97}&\CBLG{92.63$\pm$4.50}&92.51$\pm$2.98\cr
FLCS-Len&82.65$\pm$3.62&78.39$\pm$2.81&71.71$\pm$3.80&14.22$\pm$6.48&8.31$\pm$2.91\cr
\bottomrule
\end{tabular}

\end{table}

\begin{table}
\centering
\caption{Average classification accuracy on graphs with only discrete vertex label (including degree of node)}
\label{Tab:acc1c}
\begin{tabular}{lcccc}
\toprule
\!\!M{\scriptsize ETHOD}\!\!&IMDB-B&IMDB-M&MSRC-21C&ENZYMES\cr
\midrule
WLK&\CBLG{72.77$\pm$4.31}&50.44$\pm$3.72&84.66$\pm$6.92&53.16$\pm$5.93\cr
SPK&57.92$\pm$5.04&39.3$\pm$3.43&84.66$\pm$6.92&41.30$\pm$4.88\cr
WL-OA&\CBLG{72.87$\pm$3.93}&50.33$\pm$3.82&85.24$\pm$6.59&\CBLG{60.23$\pm$5.49}\cr
\midrule
WWL&72.75$\pm$4.01&50.64$\pm$4.24&83.95$\pm$6.39&57.4$\pm$6.62\cr
FGW&50.00$\pm$0.00&33.33$\pm$0.00&13.88$\pm$1.46&16.66$\pm$0.00\cr
\midrule
AWE&70.88$\pm$4.35&47.21$\pm$3.90&24.93$\pm$7.75&28.65$\pm$5.22\cr
\midrule
GIN&\CBLG{74.88$\pm$4.16}&\CBLG{51.24$\pm$4.13}&\CBLG{90.13$\pm$6.06}&47.00$\pm$5.42\cr
\midrule
BLCS&72.13$\pm$4.03&\CBLG{50.68$\pm$4.18}&86.82$\pm$6.16&\CBLG{62.60$\pm$5.52}\cr
FLCS&72.43$\pm$4.35&50.51$\pm$4.11&86.11$\pm$6.45&56.31$\pm$5.74\cr
FLCS-R&71.54$\pm$3.93&48.01$\pm$3.96&\CBLG{88.16$\pm$6.93}&31.71$\pm$5.44\cr
LevenshteinK&72.33$\pm$3.84&\CBLG{51.31$\pm$4.19}&\CBLG{87.15$\pm$6.71}&\CBLG{60.36$\pm$6.46}\cr
FLCS-Len&58.39$\pm$4.66&39.94$\pm$3.29&12.61$\pm$5.99&24.58$\pm$4.44\cr
\bottomrule
\end{tabular}
\vspace*{0.3cm}

\centering
\caption{Average classification accuracy on dense graph datasets}
\label{Tab:acc2}
\begin{tabular}{lccc}
\toprule
M{\scriptsize ETHOD}&BZR-MD&COX2-MD&ER-MD\cr
\midrule
WLK&68.78$\pm$7.47&59.77$\pm$8.40&67.22$\pm$5.75\cr
SPK&70.14$\pm$6.62&\CBLG{66.18$\pm$9.48}&62.61$\pm$6.37\cr
WL-OA&67.90$\pm$7.34&61.71$\pm$8.45&70.45$\pm$5.69\cr
\midrule
WWL&66.25$\pm$7.32&60.59$\pm$7.45&68.41$\pm$6.41\cr
FGW&\CBLG{72.06$\pm$6.52}&51.15$\pm$1.29&66.60$\pm$6.22\cr
\midrule
AWE&66.89$\pm$7.28&62.07$\pm$7.71&70.44$\pm$6.32\cr
\midrule
GIN&\CBLG{70.43$\pm$6.54}&\CBLG{66.44$\pm$8.03}&\CBLG{74.75$\pm$5.25}\cr
\midrule
BLCS&69.87$\pm$6.83&65.27$\pm$7.40&\CBLG{71.42$\pm$6.70}\cr
FLCS&69.87$\pm$6.83&65.27$\pm$7.40&69.58$\pm$6.56\cr
FLCS-R&\CBLG{72.16$\pm$7.30}&65.09$\pm$8.40&70.75$\pm$6.82\cr
LevenshteinK&69.82$\pm$7.26&\CBLG{67.05$\pm$7.70}&\CBLG{73.78$\pm$6.08}\cr
FLCS-Len&51.30$\pm$0.98&51.15$\pm$1.29&59.41$\pm$0.69\cr
\bottomrule
\end{tabular}
\end{table}

\subsection{Classification for a large-scale dataset}
To evaluate the robustness of our proposed methods and to demonstrate their performance with a large-scale dataset compared to the performance achieved by {other Graph Kernel methods and} Graph Neural Networks, we perform classification experiments on the ogbg-molhiv and {UACC257} datasets. In this experiment, we test the FLCS kernel, WWL kernel, WL kernel (WLK), {WL-OA kernel (WL-OA)} and shortest path kernel (SPK). We also add two baselines of GNN including the GIN and the Directional Graph Networks (DGN) \cite{beaini2020directional}. {For the ogbg-molhiv dataset, where these two GNN methods have already been evaluated in existing works \cite{beaini2020directional, hu2020open}, we use the public evaluation results on the website of Open Graph Benchmark. For UACC257 datasets, we conduct evaluations based on source code provided by the Open Graph Benchmark website with the default setting. }

For the parameter setting, we set fixed parameters for each kernel method empirically based on the best parameters adjusted on small datasets. Specifically, the removal ratio $\rho$ is set to 0.4. The merging radius $s$ is set to 0.5 for FLCS. Also, parameter $H$ is set to 4 for WWL, WL, and WL-OA. For parameters of EnsembleSVM, we train 10 weak classifiers on 10 subsets of training datasets. Each subset includes 1000 positive samples and 1000 negative samples, which are selected by bootstrapping. The parameter $C$ of EnsembleSVM is adjusted within $\{1,10,100,1000\}$. The ogbg-molhiv provides a fixed division of train, validation and test subsets. {For the UACC257 dataset, we manually divide it using a scaffold splitter. The training, validation, and test subsets respectively account for 60\%, 20\% and 20\% of the whole dataset.} We use the training subset for model training and use the validation subset for parameter adjustment, and the test subset for evaluation. We perform {10} runs of evaluation for each method and then compute the average test and validation ROC-AUC, and their standard deviation.

As shown in Table \ref{Tab:accogb}, where the top {three} accuracies are shown in bold typeface, {FLCS outperforms all other compared kernel methods on both datasets. It even outperforms GIN and finally reaches the second place on ogbg-molhiv dataset. For the ogbg-molhiv dataset, FLCS also shows equivalent performance to DGN, which is the best GNN method in the leaderboard (2020) for the same dataset on the website of Open Graph Benchmark.} The results demonstrate that the FLCS cooperates well with ensemble learning methods such as EnsembleSVM on large-scale datasets. 

\subsection{Effects of different sampling order for FLCS}
{The adjacent point merging operation defined in {\bf Definition \ref{Def:APM}} is actually a random merging of the distribution, which means that if one samples points in a different order, the local centers of merged distribution can be expected to change slightly. However, given the premise that only sampling points with a distance less than a properly small merging radius will be merged, local centers might change in different sampling orders, but the new local center will be a similar point to the original one. To assess effects produced by different sampling orders, one can shuffle the list of sampling points randomly to ensure that sampling points are merged in different orders, Table \ref{Tab:rand} presents accuracies of FLCS with five shuffled sampling point lists. The classification accuracy appears to change slightly. For the standard deviation of different random states, it becomes 1.32 for the MUTAG dataset and 0.45 for the IMDB-B dataset, which means that the effects of different sampling orders are sufficiently small as to be negligible.}

\begin{table}
\centering
\caption{Average ROC-AUC on large datasets}
\label{Tab:accogb}
{
\begin{tabular}{lcc}
\toprule
{M{\scriptsize ETHOD}}&ogbg-molhiv&UACC257\cr
\midrule
WWL&0.5003$\pm$0.0772&0.5469$\pm$0.0570\cr
WLK&0.7043$\pm$0.0092&0.6987$\pm$0.0638\cr
WL-OA&Out of Memory&0.7042$\pm$0.0271\cr
SPK&0.6577$\pm$0.0399&0.6193$\pm$0.0072\cr
GIN&0.7558$\pm$0.0140&\CBLG{0.8764$\pm$0.0045}\cr
DGN&\CBLG{0.7970$\pm$0.0097}&\CBLG{0.8341$\pm$0.0187}\cr
\midrule
FLCS&\CBLG{0.7808$\pm$0.0028}&\CBLG{0.7840$\pm$0.0027}\cr
LevenshteinK&\CBLG{0.7786$\pm$0.0122}&0.7712$\pm$0.0040\cr
FLCS-R&0.4386$\pm$0.0411&0.6400$\pm$0.0501\cr
\bottomrule
\end{tabular}
}

\caption{Elapsed time (Sec) for Gramian matrix computing of OT-based methods}
\label{Tab:tim}
\begin{tabular}{lcccc}
\toprule
M{\scriptsize ETHOD}&MUTAG&COX2-MD&IMDB-B&ENZYMES\cr
\midrule
WWL&25.44&68.33&373.38&249.75\cr
FGW&52.07&122.51&956.12&359.43\cr
\midrule
BLCS&600.19&101.68&1061.12&20191.16\cr
FLCS&38.40&100.79&510.43&428.00\cr
\bottomrule
\end{tabular}

\caption{Average classification accuracy of FLCS in different sampling orders}
\label{Tab:rand}
\begin{tabular}{ccc}
\toprule
Random state&MUTAG&IMDB-B\cr
\midrule
20&85.73&72.76\cr
21&84.17&71.80\cr
22&86.33&71.40\cr
23&87.78&71.72\cr
24&84.36&71.85\cr
\midrule
Average&85.67$\pm$1.32&71.90$\pm$0.45\cr
\bottomrule
\end{tabular}

\end{table}

\subsection{Effects of negative eigenvalues in a Gramian matrix for LCS kernels}
{As described in an earlier section, the LCS Graph Kernel performance might be affected by negative eigenvalues of the Gramian matrix. Therefore, we set up an experiment where we compare the average accuracies of LCS Graph Kernels with LibSVM and Kre{\u\i}n SVM (KSVM). For classification with KSVM, we first perform eigenvalue decomposition on the Gramian matrix computed by the kernel function, where the Gramian matrix $\mat{G}\in \mathbb{R}^{M \times M}$ is decomposed into a diagonal of eigenvalues $\mat{W}\in \mathbb{R}^{M \times M}$ and a matrix of $M$ normalized eigenvectors $\mat{V}\in \mathbb{R}^{M \times M}$, where each column of $\mat{V}$ is a single eigenvector. Subsequently, we replace all negative eigenvalues in $\mat{W}$ with $0$ to form a new $\mat{W}^\prime$. Then we revert the Gramian matrix by $\mat{G}^\prime = \mat{V}\mat{W}^\prime\mat{V}^T$, which will be inputted to LibSVM. Table \ref{Tab:nege} shows the average accuracies of BLCS and FLCS with LibSVM and KSVM on different datasets, from which it is apparent that the negative eigenvalues in the Gramian matrix of LCS kernels have no strong effect on their classification performance.}

\begin{table}
\centering
\caption{Difference of average classification accuracy with LibSVM and Kre{\u\i}n SVM }
\label{Tab:nege}
\begin{tabular}{lcccc}
\toprule
\multirow{2}*{Dataset}&\multicolumn{2}{c}{BLCS}&\multicolumn{2}{c}{FLCS}\cr
\cline{2-5}
 &LibSVM&KSVM&LibSVM&KSVM\cr
\midrule
MUTAG&85.54$\pm$8.21&85.69$\pm$6.97&85.80$\pm$6.54&85.18$\pm$7.23\cr
PTC-FM&61.63$\pm$5.79&61.63$\pm$5.94&64.32$\pm$6.24&64.41$\pm$5.65\cr
PTC-MR&59.70$\pm$7.03&60.52$\pm$7.15&62.49$\pm$7.36&63.22$\pm$7.15\cr
BZR&86.15$\pm$4.18&85.69$\pm$4.47&86.52$\pm$4.11&86.55$\pm$4.00\cr
COX2&83.53$\pm$4.85&83.08$\pm$5.02&79.95$\pm$3.87&80.31$\pm$4.00\cr
PROTEINS&\multicolumn{2}{c}{Out of Time}&74.88$\pm$3.72&75.87$\pm$4.03\cr
IMDB-B&72.13$\pm$4.03&73.89$\pm$4.16&72.43$\pm$4.35&73.39$\pm$3.65\cr
IMDB-M&50.68$\pm$4.18&55.01$\pm$4.04&50.51$\pm$4.11&54.95$\pm$4.09\cr
ENZYMES&62.60$\pm$5.52&62.91$\pm$5.81&56.31$\pm$5.74&56.30$\pm$5.76\cr
MSRC-9&91.23$\pm$4.71&91.23$\pm$4.31&91.23$\pm$4.71&91.41$\pm$4.78\cr
MSRC-21&92.69$\pm$3.17&92.99$\pm$3.01&92.94$\pm$3.09&92.78$\pm$3.07\cr
MSRC-21C&86.82$\pm$6.16&86.92$\pm$6.25&86.11$\pm$6.45&86.01$\pm$6.52\cr
AIDS&\multicolumn{2}{c}{Out of Time}&99.38$\pm$0.47&99.35$\pm$0.53\cr
BZR-MD&69.87$\pm$6.83&69.58$\pm$6.94&69.87$\pm$6.83&71.53$\pm$7.58\cr
COX2-MD&65.27$\pm$7.40&68.43$\pm$8.88&65.27$\pm$7.40&68.43$\pm$8.88\cr
ER-MD&71.42$\pm$6.70&74.41$\pm$6.58&69.58$\pm$6.56&73.87$\pm$6.24\cr
\bottomrule
\end{tabular}
\end{table}

\subsection{Classification under different parameters}
\begin{figure}
\centering
\subfigure[Acc./$s$ graph on MUTAG]{
\includegraphics[width=0.4\textwidth]{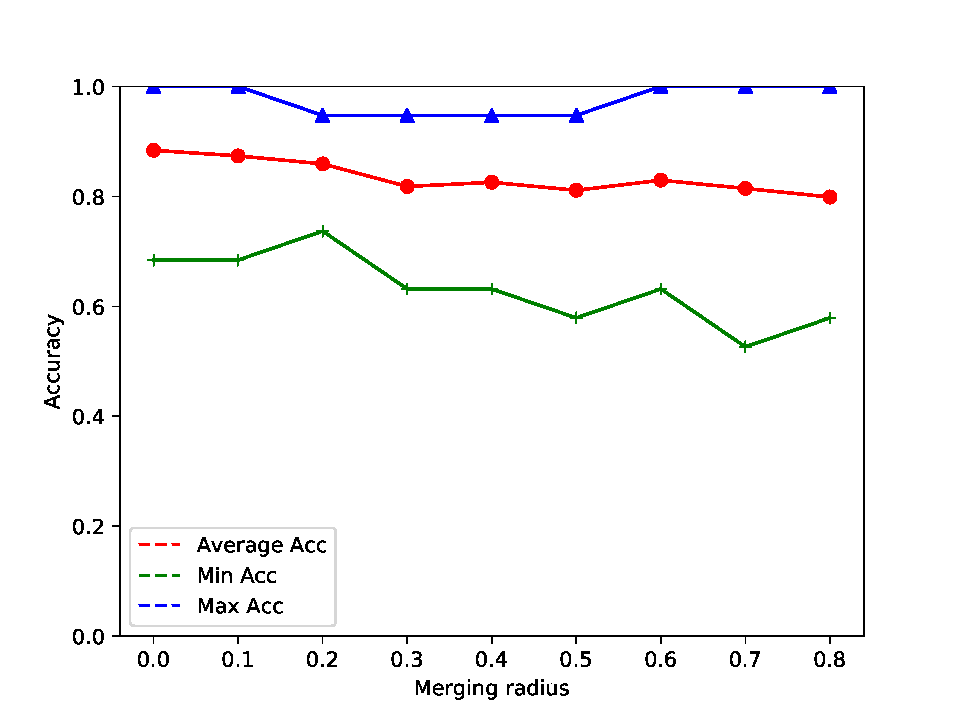} 
}
\subfigure[Acc./$s$ graph on PTC-MR]{
\includegraphics[width=0.4\textwidth]{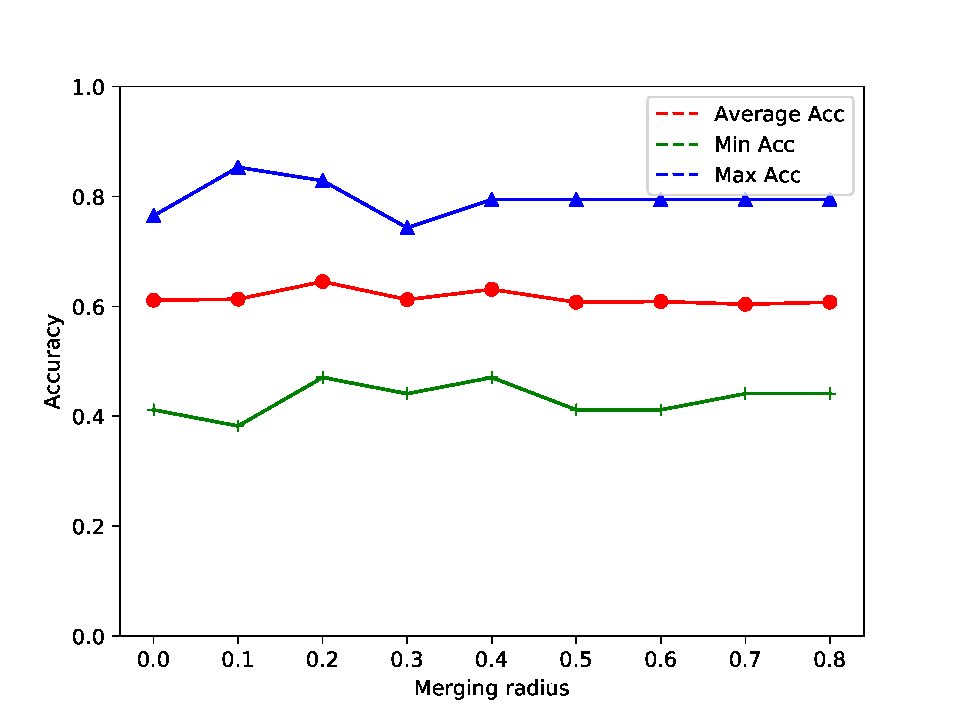} 
}\\
\subfigure[Acc./$\rho$ graph on MUTAG]{
\includegraphics[width=0.4\textwidth]{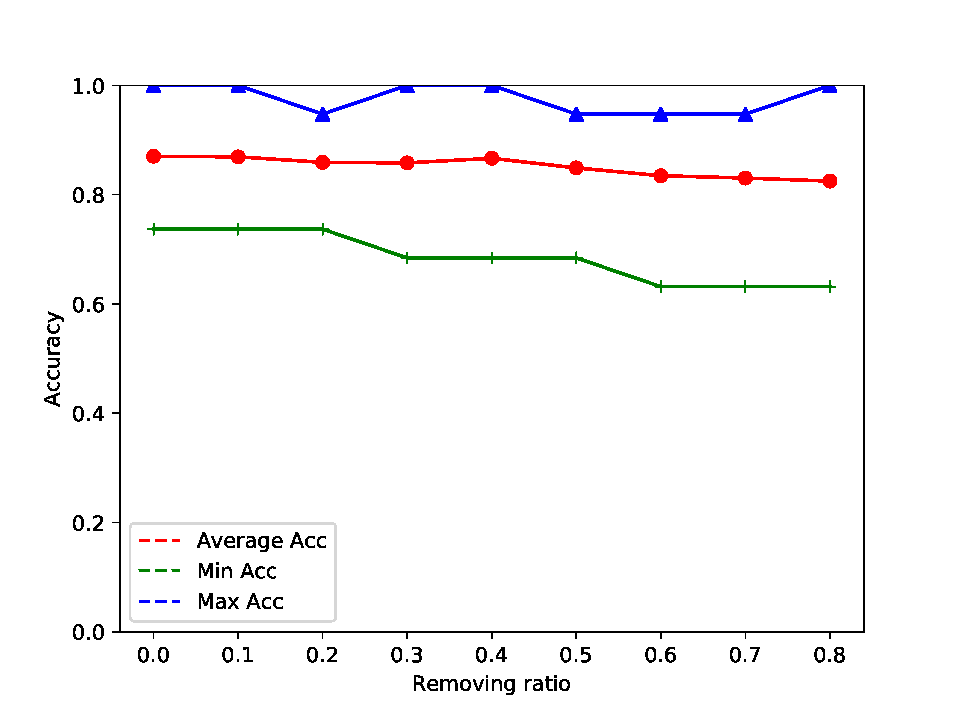} 
}
\subfigure[Acc./$\rho$ graph on PTC-MR]{
\includegraphics[width=0.4\textwidth]{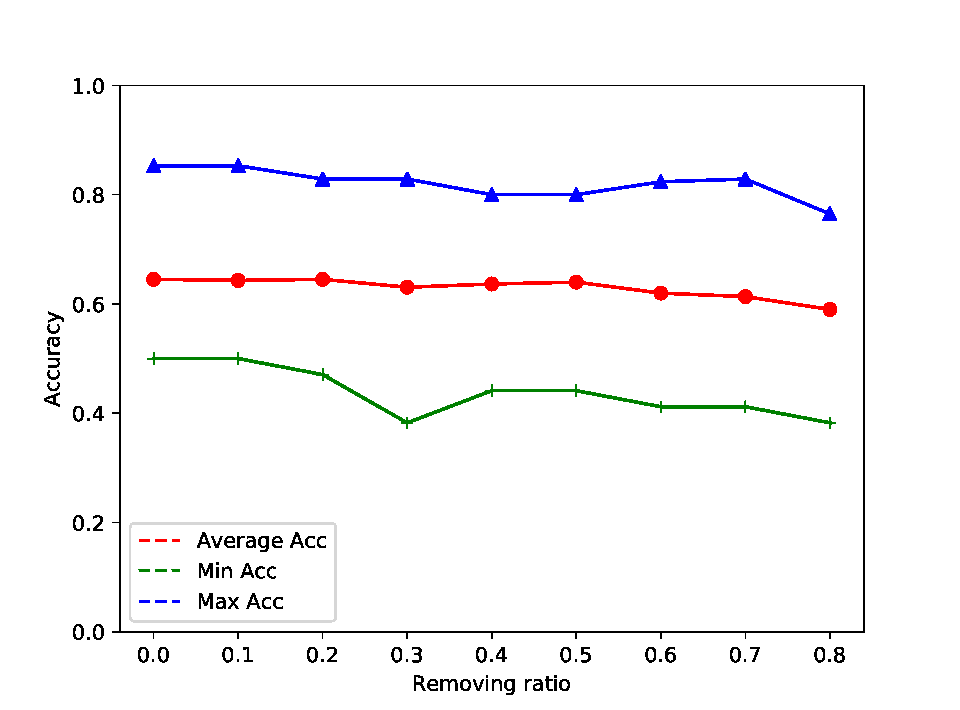} 
}\\
\subfigure[Time/$s$ graph on PTC-MR]{
\includegraphics[width=0.4\textwidth]{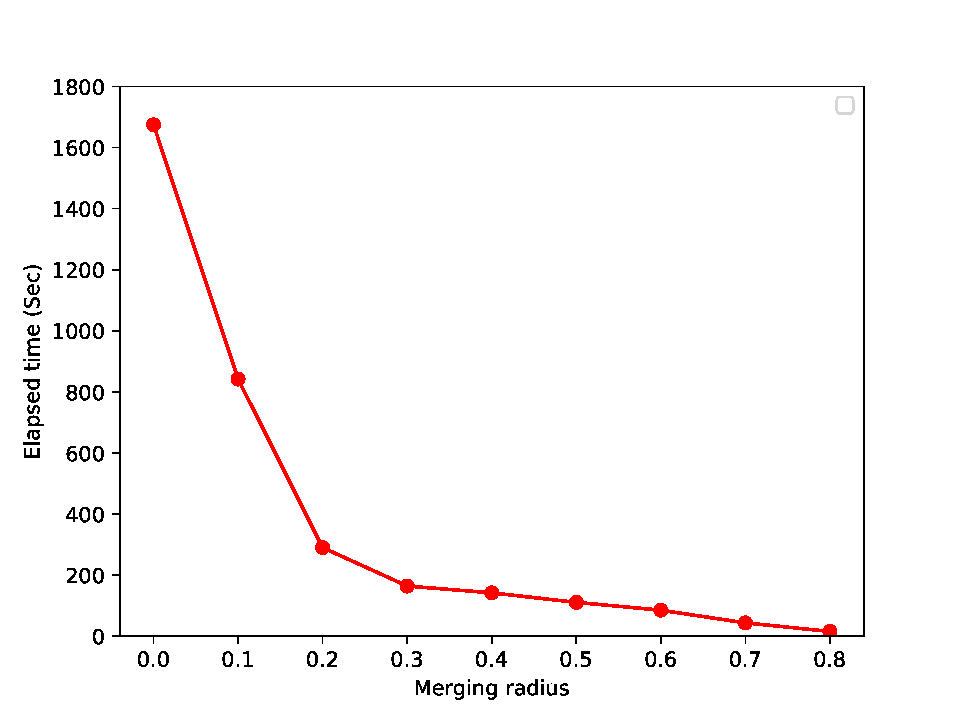} 
}
\subfigure[Time/$\rho$ graph on PTC-MR]{
\includegraphics[width=0.4\textwidth]{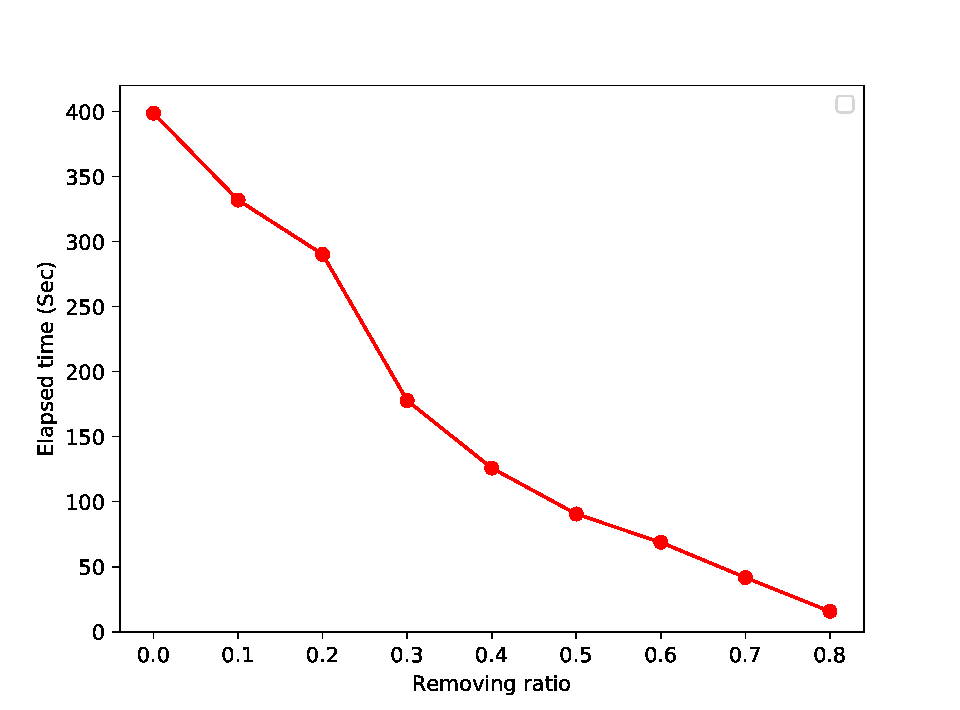} 
}
\caption{Classification accuracy and elapsed time of the LCS kernel on MUTAG and PTC-MR. For an evaluation indicator, (a), (b), (c), and (d) show tests of the classification accuracy; (e) and (f) show the elapsed time of Gramian matrix computing. For the parameters, (a), (b) and (e) are tests of different merging radius $s$ with removal ratio $\rho$ fixed at 0.2; (c), (d) and (f) are tests of different removal ratios $\rho$ with the merging radius $s$ fixed at 0.2.} 
\label{Fig:extra1}
\end{figure}
To evaluate the classification performance of FLCS, the fast variant of the proposed method under different pairwise parameters, we set up several parameter adjustment experiments. In the experiments, we change one parameter and leave the other unchanged to see how a single parameter influences the classification effect and computational cost. The classification accuracy and elapsed times of different parameters are shown in Figure \ref{Fig:extra1}.

First, we address the parameter of merging radius $s$, which controls the effects of adjacent point merging operations. If the distance between two points in LCS metric space is less than $s$, they will be merged as a local center, which means that, as the value of $s$ increases, fewer points will remain after adjacent point merging. This appears at Line 9 in {\bf Algorithm 3} in the Algorithm section of Supplementary Materials, and in {\bf Definition 7} in the paper. For classification results, because of the information loss, the accuracy might decrease as $s$ increases. However, from Figures \ref{Fig:extra1}(a) and \ref{Fig:extra1}(b), it might be readily apparent that the effect of changing $s$ differs for different datasets. Particularly, the accuracy on the MUTAG dataset shows a downward trend as the merging radius getting larger, but it does not change a lot on the PTC-MR dataset, with two local extreme points at 0.2 and 0.4. That finding shows that one need not always keep $s$ near a great value, but it must be set at a local extreme point where computational cost can be maximally reduce and considerable accuracy kept. For the elapsed time, as we had expected, when the merging radius gets larger, it decreases exponentially (from Figure \ref{Fig:extra1}(e)). At a certain point (e.g. near $s = 0.2$ on PTC-MR), adjacent point merging operation can apparently remove most of the repeated path sequences. If one continues to increase $s$, then the reduction of computational cost might not be so effective.

Next, we evaluate the influence of the removal ratio parameter $\rho$, which means that sequences having length less than $\rho$ percent of the length of the longest sequence will be removed before adjacent point merging operation. This point also appears in Line 4 in {{\bf Algorithm A.3} of the Supplementary Materials}. The classification accuracy in Figure \ref{Fig:extra1}(f) shows a downward trend as the value of $\rho$ increases. An overly large value of $\rho$ is not recommended, but an appropriate value such as 0.2 can be set, which can maintain the accuracy and reduce computational cost maximally.

\section{Conclusion}
We have proposed a novel method of computing similarity between graphs using the Wasserstein distance in LCS metric space. Comparing to other state-of-the-art path-based graph kernels, our method {enables paths with different lengths to be compared} and can give a more accurate metric between graphs. We present both basic and fast LCS Graph Kernels for graph classification tasks, among which the fast LCS Graph Kernel successfully reduces the computational cost of {a huge number of path comparisons}. Although LCS kernels use only length information of LCS, numerical experiments reveal that it shows better performance than those of others. {Moreover, the evaluation result for large-scale datasets also reveals that our proposed methods are competitive with some state-of-the-art GNN methods, even for a large dataset}. As future work, we intend to examine more complicated information of LCS specifically, such as subsequence patterns of paths, which might help us to capture the key classification standard in specific classification problems. Moreover, common subsequence patterns of paths or walks are more appropriate as neural network inputs. They can help build a model that includes common features with respect to specific classification problems and which can embed a graph directly rather than comparing each graph in training data.

\clearpage
\appendix
\setcounter{section}{0}
\renewcommand{\thesection}{A}
\renewcommand{\thealgorithm}{A.\arabic{algorithm}} 

\section{Supplementary Materials}

\subsection{Proof of Theorem 1}
In this subsection, we will prove some important propositions about LCS metric space. 

We give a self-contained proof of {\bf Theorem 1} by exactly following the results of previous work \cite{bakkelund2009lcs}, which also proposes an LCS-based string metric similar to ours. We first restate the definition of metric before proceeding.
\begin{definition}[Metric]
A metric on the set $\mathcal{X}$ is a function $d: \mathcal{X}\times \mathcal{X} \to \mathbb{R}$ that satisfies the following properties for all $x, y, z \in \mathcal{X}$:

(1) Positive definiteness: $d(x, y) \geq 0$, and $d(x,y) = 0$ if and only if $x = y$;

(2) Symmetry: $d(x, y) = d(y, x)$;

(3) Triangle inequality: $d(x,y) \leq d(x,z) + d(z,y)$
\end{definition}
To prove the LCS metric of LCS metric space $d(\vec{x}_1, \vec{x}_2) = 1 - \frac{F_{\rm lcs}(\vec{x}_1, \vec{x}_2)}{\max (|\vec{x}_1|, |\vec{x}_2|)}$ (Equation (2) and (3) in the paper) is a metric, we only have to prove the following lemmas:

\begin{lemma}
\label{Lem:pds}
The function of LCS metric is positive definite and symmetric.
\end{lemma}

\begin{proof}Because {the} longest common subsequence is a subsequence of its parent sequences, it means that $\frac{F_{\rm lcs}(\vec{x}_1, \vec{x}_2)}{\max (|\vec{x}_1|, |\vec{x}_2|)} \leq 1$ and $F_{\rm lcs}(\vec{x}_1, \vec{x}_2) = \max (|\vec{x}_1|, |\vec{x}_2|)$ if and only if $\vec{x}_1 = \vec{x}_2$. Therefore, $d(\vec{x}_1, \vec{x}_2) \geq 1 - 1= 0$, and $d(\vec{x}_1,\vec{x}_2) = 0$ if and only if $\vec{x}_1 = \vec{x}_2$. Thus the function of LCS metric is positive definite.

As for the symmetry, because {the} longest common subsequence method does not care about the order of compared parent sequences, we have $d(\vec{x}_1, \vec{x}_2) = d(\vec{x}_2, \vec{x}_1)$. Subsequently, the function of LCS metric is symmetric. This completes the proof.
\end{proof}

\begin{lemma}
\label{Lem:ti}
The function of LCS metric satisfies the triangle inequality.
\end{lemma}
Before proceeding, we first restate an important proposition of {LCS proven} by Bakkelund \cite{bakkelund2009lcs}.

\begin{proposition}[{\bf Lemma 3.3} and {\bf Corollary 3.4 } in \cite{bakkelund2009lcs}]
\label{pro:INE}
The following inequality holds for all $\vec{x},\vec{y},\vec{z} \in \mathcal{X}$, where $\mathcal{X}$ is a certain sequence set:
\begin{eqnarray*}
F_{\rm lcs}(\vec{x},\vec{y}) + F_{\rm lcs}(\vec{y},\vec{z}) - F_{\rm lcs}(\vec{x},\vec{z}) \leq |\vec{y}|.
\end{eqnarray*}
\end{proposition}
\begin{proof}

Because the symmetry of LCS metric $d(\vec{x}, \vec{y})$ is {proven} in {\bf Lemma \ref{Lem:pds}}, we can safely assume that $|\vec{x}| \leq |\vec{z}|$ in the triangle inequality as follows:
\begin{eqnarray}
\label{Eq:TINE}
d(\vec{x},\vec{z}) \leq d(\vec{x},\vec{y}) + f(\vec{y},\vec{z}),
\end{eqnarray}
where $\vec{x},\vec{y},\vec{z} \in \mathcal{X}$. Thus we have $M(\vec{x},\vec{y}) \leq M(\vec{y},\vec{z})$ and $M(\vec{x},\vec{z})\leq M(\vec{y},\vec{z})$, where $M(\vec{x},\vec{y}) = \max (|\vec{x}|, |\vec{y}|)$. Due to {\bf Proposition \ref{pro:INE}}, we can derive:
\begin{eqnarray*}
M(\vec{x},\vec{y})(F_{\rm lcs}(\vec{y},\vec{z}) - F_{\rm lcs}(\vec{x},\vec{z})) &\leq \cr
M(\vec{y}, \vec{z})(|\vec{y}|-F_{\rm lcs}(\vec{x}, \vec{y}))&\ 
\end{eqnarray*}
$\Longleftrightarrow$
\begin{eqnarray*}
F_{\rm lcs}(\vec{x},\vec{y})M(\vec{y},\vec{z})+F_{\rm lcs}(\vec{y},\vec{z})M(\vec{x},\vec{y}) &\leq \cr
|\vec{y}|M(\vec{y}, \vec{z}) + F_{\rm lcs}(\vec{x},\vec{z})M(\vec{x},\vec{y}).&\ 
\end{eqnarray*}
Because we have $M(\vec{x},\vec{z})\leq M(\vec{y},\vec{z})$, we can derive that:
\begin{eqnarray*}
F_{\rm lcs}(\vec{x},\vec{y})M(\vec{y},\vec{z})+F_{\rm lcs}(\vec{y},\vec{z})M(\vec{x},\vec{y}) &\leq \cr
M(\vec{x}, \vec{y})M(\vec{y}, \vec{z}) + F_{\rm lcs}(\vec{x},\vec{z})M(\vec{x},\vec{y})\frac{M(\vec{y},\vec{z})}{M(\vec{x},\vec{z})}.&\ 
\end{eqnarray*}
Then we can obtain the following by dividing both sides of inequality above by $M(\vec{x},\vec{y})M(\vec{y},\vec{z})$:
\begin{eqnarray*}
\frac{F_{\rm lcs}(\vec{x},\vec{y})}{M(\vec{x},\vec{y})}+\frac{F_{\rm lcs}(\vec{y},\vec{z})}{M(\vec{y},\vec{z})} \leq 1 + \frac{F_{\rm lcs}(\vec{x},\vec{z})}{M(\vec{x},\vec{z})}
\end{eqnarray*}
$\Leftrightarrow$
\begin{eqnarray*}
1 - \frac{F_{\rm lcs}(\vec{x},\vec{z})}{M(\vec{x},\vec{z})} \leq 1 - \frac{F_{\rm lcs}(\vec{x},\vec{y})}{M(\vec{x},\vec{y})}+1 - \frac{F_{\rm lcs}(\vec{y},\vec{z})}{M(\vec{y},\vec{z})}.
\end{eqnarray*}

Thus Eq. (\ref{Eq:TINE}) holds for all $\vec{x},\vec{y},\vec{z} \in \mathcal{X}$, and the length between any two of them is not equal to zero. Actually, we do not consider the situation that a sequence is empty, because there exists no shortest path of which length is equal to zero.

\end{proof}

Consequently, the proof of {\bf Theorem 1} is finally given as presented below:
\begin{proof}
From {\bf Lemmas \ref{Lem:pds}}, the function of LCS metric is positive definite and symmetric. Also, 
{\bf Lemmas \ref{Lem:ti}} guarantees that the function of LCS metric satisfies the triangle inequality. Therefore, LCS metric satisfies the axiom for a metric. This completes the proof. 
\end{proof}

\subsection{Algorithm}
In this section, we will present the algorithm details of some key operations in the BLCS kernel and the FLCS kernel.

{\bf Algorithm \ref{Alg.1}} shows the implementation details of path sequence set computation. The ${\rm concatenate}(\vec{x},\vec{y})$ function here  concatenating two sequences $\vec{x},\vec{y}$ together using a push-back operation.

{\bf Algorithm \ref{Alg.2}} shows the implementation details of graph representation computation in {the} BLCS kernel, which outputs the path sequence set $\mathcal{X}$ and its mass vector $\mathbf{m}$. $\mathbf{m}(k)$ denotes the $k$-th element of $\mathbf{m}$.

{\bf Algorithm \ref{Alg.4}} shows the implementation details of the improved way to compute the path sequence set $\mathcal{X}$ and mass vector
$\mathbf{m}$ by performing fragmented path removing and adjacent point merging.

{\bf Algorithm \ref{Alg.3}} shows the implementation of LCS kernel value computation. We here use {the S}inkhorn method 
to solve optimal transport problem with ground distance matrix $\mat{D}$. We also use the Laplacian Kernel Function with parameter $\lambda$ to calculate kernel value.

\begin{algorithm}[H]
\caption{Compute shortest path serialization}
\label{Alg.1}
\begin{algorithmic}[1]
\Require Shortest path $p^\star_{i,j}$ in graph $\mathcal{G}=\{\mathcal{V},\mathcal{E}\}$; vertices $v_i \in \mathcal{V}$; edges of $v_i$ and $v_j$: 
$e_{i,j} \in \mathcal{E}$; mapping function of categorical node label $l(v)$; mapping function of categorical edge label $w(e)$.
\Ensure Shortest path sequence $\vec{x}_{i,j}$.
\State	$\vec{x}_{i,j} \gets null$ // Initialize the path sequence of $p^\star_{i,j}$
	\For {$v_k$ \textbf{in} $p_{i,j}^\star$} // Operate each node
\State	$\vec{x}_{i,j} \gets {\rm concatenate}(\vec{x}_{i,j},l(v_k))$

	// Concatenate $\vec{x}_{i,j}$ with next node label
		\If{$v_k$ is not the last node}
\State		$\vec{x}_{i,j} \gets {\rm concatenate}(\vec{x}_{i,j},-w(e_{k,k+1}))$

		// Concatenate $\vec{x}_{i,j}$ with next edge label
		\EndIf
	\EndFor

\end{algorithmic}
\end{algorithm}

\begin{algorithm}[H]
\caption{Compute graph representation in basic method}
\label{Alg.2}
\begin{algorithmic}[1]
\Require Objective graph $G=(\mathcal{V},\mathcal{E})$.
\Ensure Path sequence set $\mathcal{X}$; mass vector $\mathbf{m}$.
\State $\mathcal{X},\mathbf{m}\gets null$ 

// Initialize the shortest path sequence set of $\mathcal{G}$
\State $\mathcal{P}\gets$ Compute shortest path set
\For {$p_{i,j}^\star$ \textbf{in} $\mathcal{P}$} 
\State	$\vec{x}_{i,j} \gets F_{\rm serialize}(p_{i,j}\star)$ 

// Serialize each shortest path
\If {$\exists \vec{x}_k \in \mathcal{X}, \vec{x}_k = \vec{x}_{i,j}$}
\State $\mathbf{m}(k) \gets \mathbf{m}(k)+1 $
\Else
\State $\mathbf{m} \gets {\rm concatenate}(\mathbf{m},1)$
\State $\mathcal{X} \gets \mathcal{X}\cup \vec{x}_{i,j}$

// Add path sequence $\vec{x}_{i,j}$ to the set
\EndIf
\EndFor
\end{algorithmic}
\end{algorithm}

\begin{algorithm}[H]
\caption{Compute graph representation in fast method}
\label{Alg.4}
\begin{algorithmic}[1]
\Require Objective graph $G=(\mathcal{V},\mathcal{E})$; removing ratio $\rho \in [0,1)$; combination radius $s \in [0,1]$.
\Ensure Path sequence set $\mathcal{X}$; mass vector $\mathbf{m}$.
\State $\mathcal{X},\mathbf{m} \gets null$ 
\State $\mathcal{P},L_{\rm max}\gets$ Compute shortest path set of $G$ and the length of the longest one
\For {$p_{i,j}^\star$ \textbf{in} $\mathcal{P}$} 
	\If {$|p_{i,j}^\star| \leq \rho L_{\rm max}$} 
\State	\textbf{continue} // Removing fragmented paths
	\EndIf
\State	$\vec{x}_{i,j} \gets F_{\rm serialize}(p_{i,j}^\star)$ 

// Serialize each shortest path
\State	$\vec{x}_k, m_k \gets$ the nearest point of $\vec{x}_{i,j}$ in $\mathcal{X}$ and its mass
		\If {$\mathcal{X} \neq null$ \textbf{and} $d(\vec{x}_k, \vec{x}_{i,j}) <= s$}
\State		$\vec{x}_k \gets$ the lonest one among $\vec{x}_k,\vec{x}_{i,j}$
\State		$\mathbf{m}(k) \gets \mathbf{m}(k)+1$ 

// Combine the masses of $\vec{x}_k,\vec{x}_{i,j}$
		\Else
			
\State	$\mathbf{m} \gets {\rm concatenate}(\mathbf{m},1)$
\State 	$\mathcal{X}\gets \mathcal{X}\cup \vec{x}_{i,j}$ // Add $\vec{x}_{i,j}$ as a new point to $\mathcal{X}$
		\EndIf
\EndFor
\end{algorithmic}
\end{algorithm}

\begin{algorithm}[H]
\caption{Compute LCS kernel value}
\label{Alg.3}
\begin{algorithmic}[1]
\Require Two objective path sequence sets $\mathcal{X}_1,\mathcal{X}_2$ {obtained from objective graphs $G_1,G_2$}; mass vectors $\mathbf{m}_1, \mathbf{m}_2$;$\lambda$.
\Ensure LCS kernel value $k_{\rm LCS}(G_1,G_2)$
\State $\mat{D} \gets null$ // Initialize cost matrix of transporting $\mathcal{X}_1,\mathcal{X}_2$
\For {$(\vec{x}_1)_{i}$ \textbf{in} $\mathcal{X}_1$} 
	\For {$(\vec{x}_2)_{j}$ \textbf{in} $\mathcal{X}_2$} 
\State	$\mat{D}(i,j) \gets 1-\frac{F_{\rm lcs}((\vec{x}_1)_{i},(\vec{x}_2)_{j})}{\max(|(\vec{x}_1)_{i}|,|(\vec{x}_2)_{j}|)}$

// Compute ground distance matrix of $(\vec{x}_1)_{i},(\vec{x}_2)_{j}$ in LCS metric space
	\EndFor
\EndFor
\State $\mat{T}^\star \gets {\rm Sinkhorn}(\mat{D},\frac{\mathbf{m}_1}{||\mathbf{m}_1||_1},\frac{\mathbf{m}_2}{||\mathbf{m}_2||_1})$

 // Use the Sinkhorn method to compute optimal transport plan $\mat{T}^\star$
\State $d_{G} \gets \langle\mat{T}^\star,\mat{D}\rangle$ // Compute the Wasserstein distance
\State $k_{\rm LCS}(G_1,G_2) \gets e^{-\lambda d_{G}}$
\end{algorithmic}
\end{algorithm}

\clearpage	
\bibliographystyle{unsrt}
\bibliography{graph,graph_datasets,optimal_transport}

\end{document}